\documentclass[10pt,twocolumn,letterpaper]{article}
\usepackage[pagenumbers]{cvpr}
\usepackage{graphicx}
\usepackage{amsmath,amssymb}
\usepackage{amsthm}
\usepackage{placeins}
\newtheorem{proposition}{Proposition}[section]
\newtheorem{corollary}{Corollary}[section]
\usepackage{microtype}
\AtBeginDocument{%
  \setlength{\abovedisplayskip}{7pt plus 2pt minus 5pt}%
  \setlength{\belowdisplayskip}{7pt plus 2pt minus 5pt}%
  \setlength{\abovedisplayshortskip}{0pt plus 3pt}%
  \setlength{\belowdisplayshortskip}{4pt plus 3pt minus 3pt}%
}
\usepackage{booktabs}
\usepackage{colortbl}
\definecolor{sgfgray}{gray}{0.93}

\definecolor{cvprblue}{rgb}{0.21,0.49,0.74}
\usepackage[pagebackref,breaklinks,colorlinks,allcolors=cvprblue]{hyperref}

\def\paperID{0000}
\def\confName{CVPR}
\def\confYear{2027}
\title{SGF+: Decoupling Gradient Flows for Autoregressive Video Generation}
\author{%
{\normalsize Zihan Su$^{1,2*}$\quad Junhao Zhuang$^{2*\dagger}$\quad Yaowei Li$^{2}$\quad Siwen Lu$^{1}$\quad Haoran Li$^{2}$\quad Lingen Li$^{3}$\quad Haoyu Wu$^{2}$}\\
{\normalsize Weiyang Jin$^{2}$\quad Songchun Zhang$^{2}$\quad Haoyang Huang$^{2}$\quad Chun Yuan$^{1\dagger}$\quad Zeyue Xue$^{2}$\quad Nan Duan$^{2}$}\\[3pt]
{\normalsize $^{1}$Tsinghua University\quad $^{2}$Joy Future Academy, JD\quad $^{3}$The Chinese University of Hong Kong}\\[2pt]
{\small $^{*}$Equal contribution.\quad $^{\dagger}$Corresponding authors.}\\[2pt]
{\normalsize \url{https://zihan-su.github.io/self-gradient-forcing-plus}}
}

\makeatletter
\g@addto@macro\@maketitle{%
  \vspace*{-26pt}%
  \begin{center}
    \includegraphics[width=\textwidth]{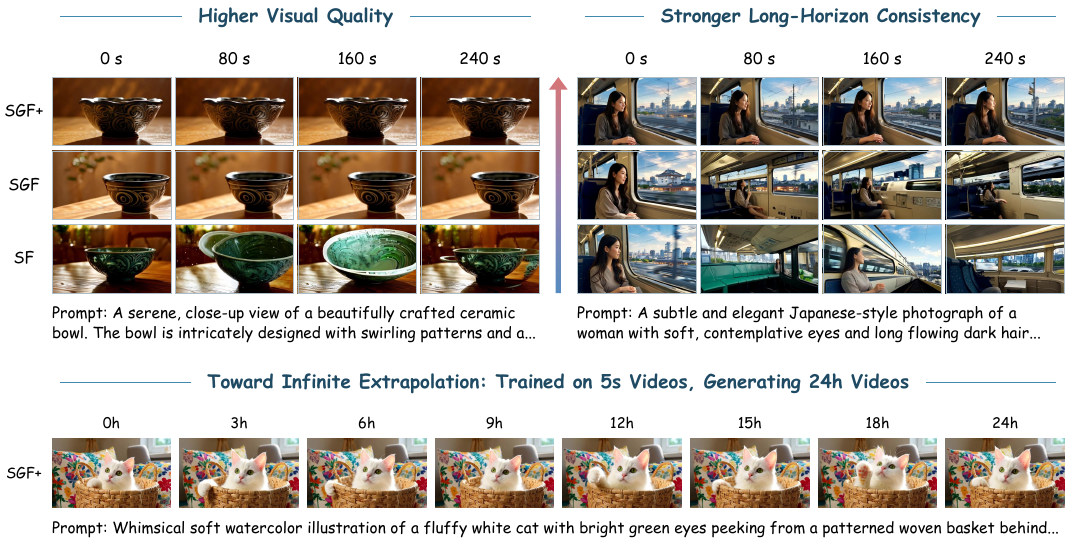}
    \vspace*{-10pt}
    \captionsetup{type=figure,skip=2pt}
    \captionof{figure}{\textbf{Qualitative comparison of autoregressive video generation methods.} Self Forcing (SF) blocks context gradients, preventing future predictions from supervising context writing and causing pronounced visual degradation over long rollouts. Self Gradient Forcing (SGF) restores context gradients and improves long-horizon consistency, but conflicts between context-writing and denoising updates keep both consistency and visual quality suboptimal. Self Gradient Forcing Plus (SGF+) decouples context writing and denoising, achieving the best visual quality and long-horizon consistency among the three methods. Trained with only 5s rollouts, SGF+ supports continuous generation for up to 24 hours without long-video fine-tuning.}
    \label{fig:teaser}
  \end{center}%
  \vspace*{6pt}%
}
\makeatother

\begin{document}
\maketitle

\begingroup
\makeatletter
\setlength{\@dblfptop}{0pt}
\setcounter{dbltopnumber}{1}
\makeatother
\begin{figure*}[t]
  \centering
  \includegraphics[width=\textwidth]{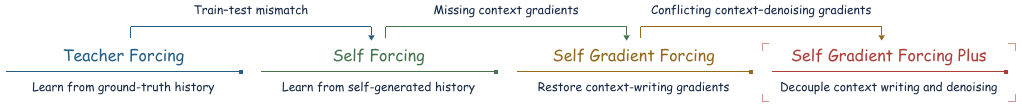}
  \caption{\textbf{Method lineage for autoregressive video generation.} SF addresses train--test history mismatch via self-generated rollouts. SGF restores context-writing gradients, while SGF+ decouples context writing from denoising to resolve their gradient conflict.}
  \label{fig:method-lineage}
\end{figure*}

\begin{abstract}
Autoregressive video generation requires denoising the current frames while writing their key-value representations as context for future predictions. However, these two roles typically share parameters, and we find that their gradients exhibit distinct patterns and systematic negative alignment, hindering the joint optimization of visual quality and temporal consistency. We introduce Self Gradient Forcing Plus (SGF+), which assigns separate parameters to context writing and denoising while preserving their interaction through causal attention. Both roles are jointly optimized using the original generation objective without auxiliary losses, with context writing supervised through its contribution to future predictions. This simple change improves visual quality and long-horizon consistency over the evaluated baselines in both framewise and chunkwise generation, without additional video training data or a longer training horizon. Trained on only 5s rollouts, SGF+ supports continuous generation for up to 24 hours without long-video fine-tuning. These results highlight role-specific parameterization as an effective design principle for high-quality autoregressive video generation and native long-horizon extrapolation.
\end{abstract}

\newcommand{\gradientfigure}{%
\begin{figure*}[t]
  \setlength{\abovecaptionskip}{4pt}
  \centering
  \begin{minipage}[t]{0.325\textwidth}
    \vspace{0pt}
    \centering
    \begin{minipage}[c][40mm][c]{\linewidth}
      \centering
      \includegraphics[width=\linewidth]{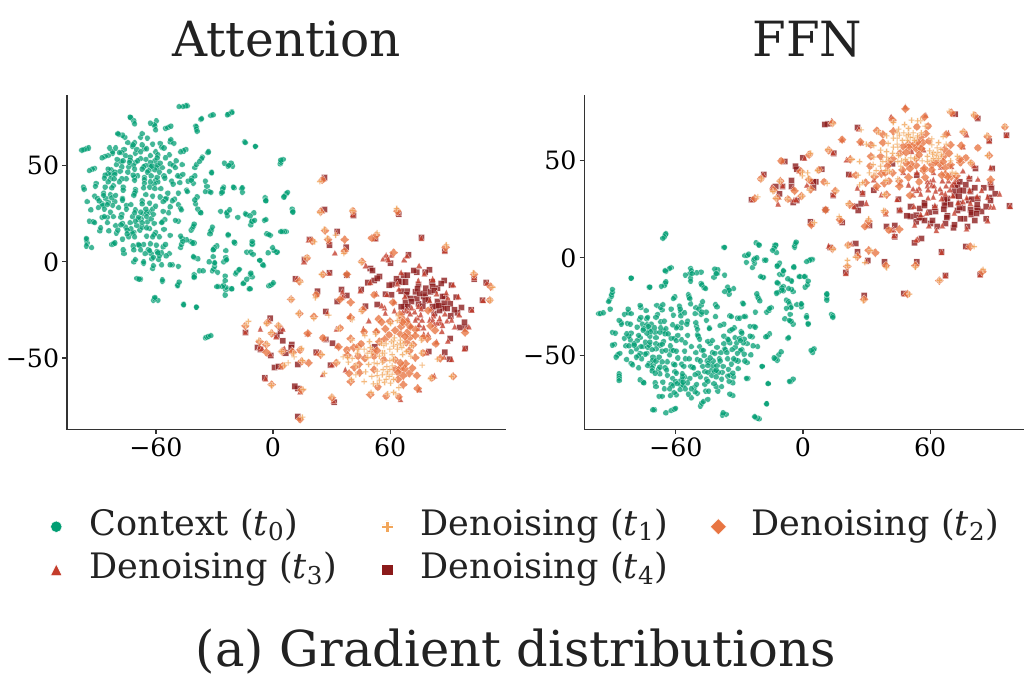}
    \end{minipage}
  \end{minipage}\hfill
  \begin{minipage}[t]{0.325\textwidth}
    \vspace{0pt}
    \centering
    \begin{minipage}[c][40mm][c]{\linewidth}
      \centering
      \includegraphics[width=\linewidth]{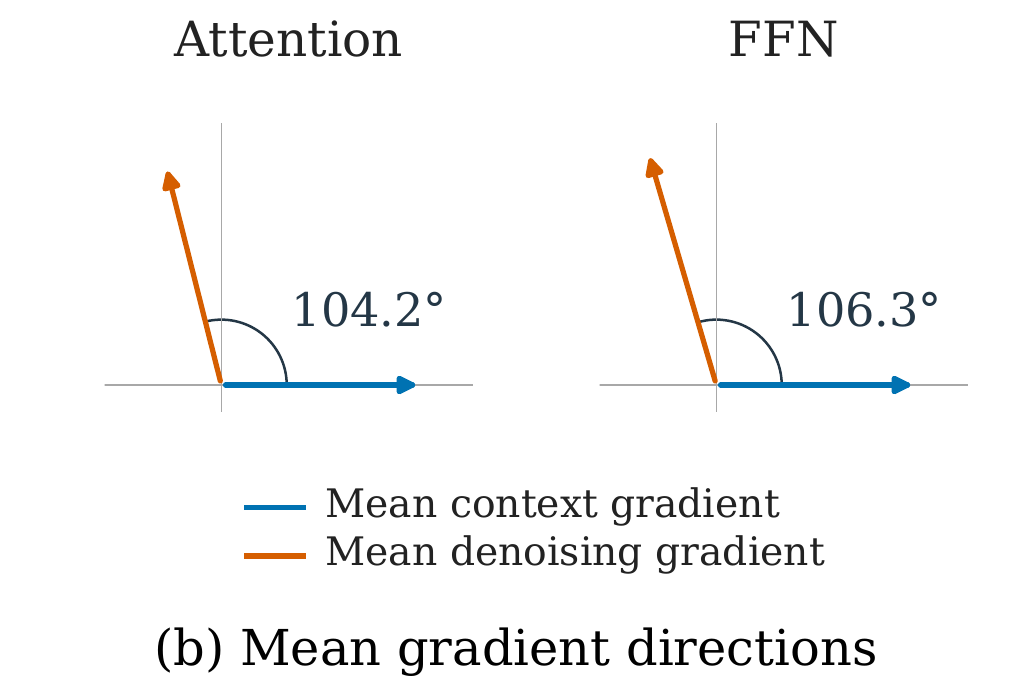}
    \end{minipage}
  \end{minipage}\hfill
  \begin{minipage}[t]{0.325\textwidth}
    \vspace{0pt}
    \centering
    \begin{minipage}[c][40mm][c]{\linewidth}
      \centering
      \includegraphics[width=\linewidth]{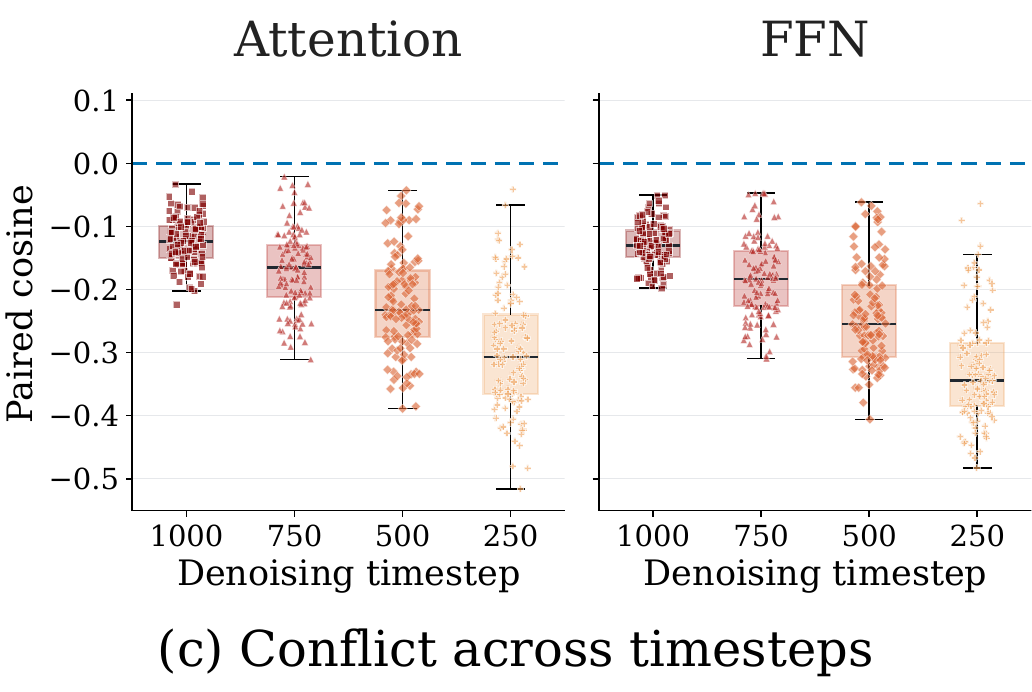}
    \end{minipage}
  \end{minipage}
  \caption{\textbf{Visualization of gradient conflict.} We analyze SGF gradients over 128 prompts and 4 denoising timesteps, separately for Attention and FFN. (a) Angular-distance t-SNE visualizes the distributions of context-writing and denoising gradients. $t_0=0$ denotes context writing at timestep 0, while $t_1,t_2,t_3,t_4$ denote denoising at timesteps 250, 500, 750, and 1000, respectively. (b) Gradients are aggregated to compute the mean directions. (c) At each denoising timestep, we compute the paired cosine similarity between the denoising and context-writing gradients for the same prompt. Layer-wise gradient distributions are provided in Appendix~\ref{sec:appendix-layers}.}
  \label{fig:gradient-conflict}
\end{figure*}
}

\section{Introduction}
\label{sec:introduction}

Autoregressive video diffusion models generate videos sequentially by denoising new frames conditioned on historical context~\cite{yin2025slowbidirectionalfastautoregressive,huang2025selfforcingbridgingtraintest,liu2025rollingforcingautoregressivelong,su2026and}. The generated frames are then encoded into key-value (KV) representations that provide context for subsequent predictions. This process supports low-latency video generation and interactive world modeling, where inference can extend far beyond the temporal horizon used in training~\cite{yang2025longliverealtimeinteractivelong, huang2026solarwmopendatascalable, savva2026solaris, zhang2026echowmopenenterableomnimodal, duan2026long}. It involves two distinct computational roles: \emph{denoising} generates the current content, while \emph{context writing} determines how that content is represented for future generation. Learning both roles is essential for maintaining visual quality and temporal consistency as generation continues.

As summarized in Figure~\ref{fig:method-lineage}, Self Forcing (SF)~\cite{huang2025selfforcingbridgingtraintest} mitigates exposure bias by training on self-generated histories with a video-level distribution-matching objective. However, to maintain computational tractability, SF treats the historical KV cache as frozen rollout state for subsequent predictions. Future generation losses can therefore supervise how noisy target tokens read history, but not how earlier generated latents are encoded into useful K/V representations for future generation. This missing supervision path constitutes the \emph{historical context-gradient gap}~\cite{zhuang2026selfgradientforcingnative}.

Self Gradient Forcing (SGF)~\cite{zhuang2026selfgradientforcingnative} closes this gap through a two-pass training procedure. The first pass performs a no-gradient autoregressive rollout, while the second reconstructs context representations and target predictions in parallel. The sampled latents remain detached, but the reconstructed KV path is differentiable, allowing future generation losses to supervise context writing without backpropagating through the full rollout. Although this improves long-video extrapolation, context writing and denoising still share the same parameters. We observe residual degradation in long sequences and, in a small number of generated videos, even lower visual quality in the initial frames than with SF. Since the first generated frame has no preceding video context, it provides a useful probe of denoising quality. These observations motivate a fundamental question: \emph{Can shared parameters effectively accommodate the optimization demands of both context writing and denoising?}

\gradientfigure

We investigate this question by analyzing the gradient contributions of context writing and denoising under the same generation objective during SGF training. As shown in Figure~\ref{fig:gradient-conflict}, (a) angular-distance t-SNE reveals distinct role-dependent gradient groups in both Attention and FFN. (b) After aggregation within each module family, the mean gradient directions exhibit large angular separation: $104.2^\circ$ in Attention and $106.3^\circ$ in FFN, both exceeding $90^\circ$. (c) Pairwise analysis across prompts and denoising timesteps further reveals systematic negative alignment, with negative cosine similarities for all 512 measured pairs in both module families. These results reveal substantially different, even conflicting, optimization demands on shared parameters, motivating role-specific parameterization. Appendix~\ref{sec:role-separation-theory} provides a local theoretical analysis of its optimization advantage under gradient conflict.

Based on this analysis, we introduce \textbf{Self Gradient Forcing Plus (SGF+)}, which revisits a basic design choice in autoregressive video forcing: sharing parameters between context writing and denoising. SGF+ adopts \emph{role-specific parameterization}, assigning independent parameters to the two roles while preserving their interaction through the causal attention structure. The \emph{context writer} encodes generated history into layer-wise KV representations, which the \emph{denoiser} reads to generate current frames. Both roles are jointly optimized using the original generation objective, with context writing supervised through its contribution to future predictions. This simple architectural change retains SGF's self-generated rollouts, differentiable replay, and short training horizon, requiring no auxiliary losses, additional video training data, or long-horizon fine-tuning. By intervening in parameter sharing, SGF+ addresses a different design dimension from history construction and context management, making these approaches complementary in mechanism.

Experiments on VBench over generation horizons of 5s, 60s, and 240s show that SGF+ achieves better visual quality and long-horizon consistency than SF and SGF in both framewise and chunkwise generation. The qualitative comparisons in Figure~\ref{fig:teaser} illustrate these gains over long rollouts. With only 5s training rollouts, SGF+ supports continuous generation for up to 24 hours, without extending the training horizon or applying long-video fine-tuning. We refer to this ability as \emph{native long-horizon extrapolation}: sustained generation learned within a limited temporal training window. Together, these results highlight \emph{role-specific parameterization} as an effective design principle for high-quality autoregressive video generation and native long-horizon extrapolation. Our main contributions are as follows:
\begin{itemize}
  \item We identify distinct, even conflicting, context-writing and denoising gradients under the same objective, revealing an optimization limitation of shared parameters.
  \item We introduce Self Gradient Forcing Plus (SGF+), a role-specific parameterization for autoregressive video diffusion that jointly trains context writing and denoising with the original generation objective, without auxiliary losses, additional video training data, or long-horizon fine-tuning.
  \item Experiments show improved visual quality and long-horizon consistency in framewise and chunkwise generation, with native long-horizon extrapolation to 24-hour continuous videos from 5s training rollouts.
\end{itemize}

\begin{figure*}[t]
  \setlength{\abovecaptionskip}{6pt}
  \centering
  \includegraphics[width=\textwidth]{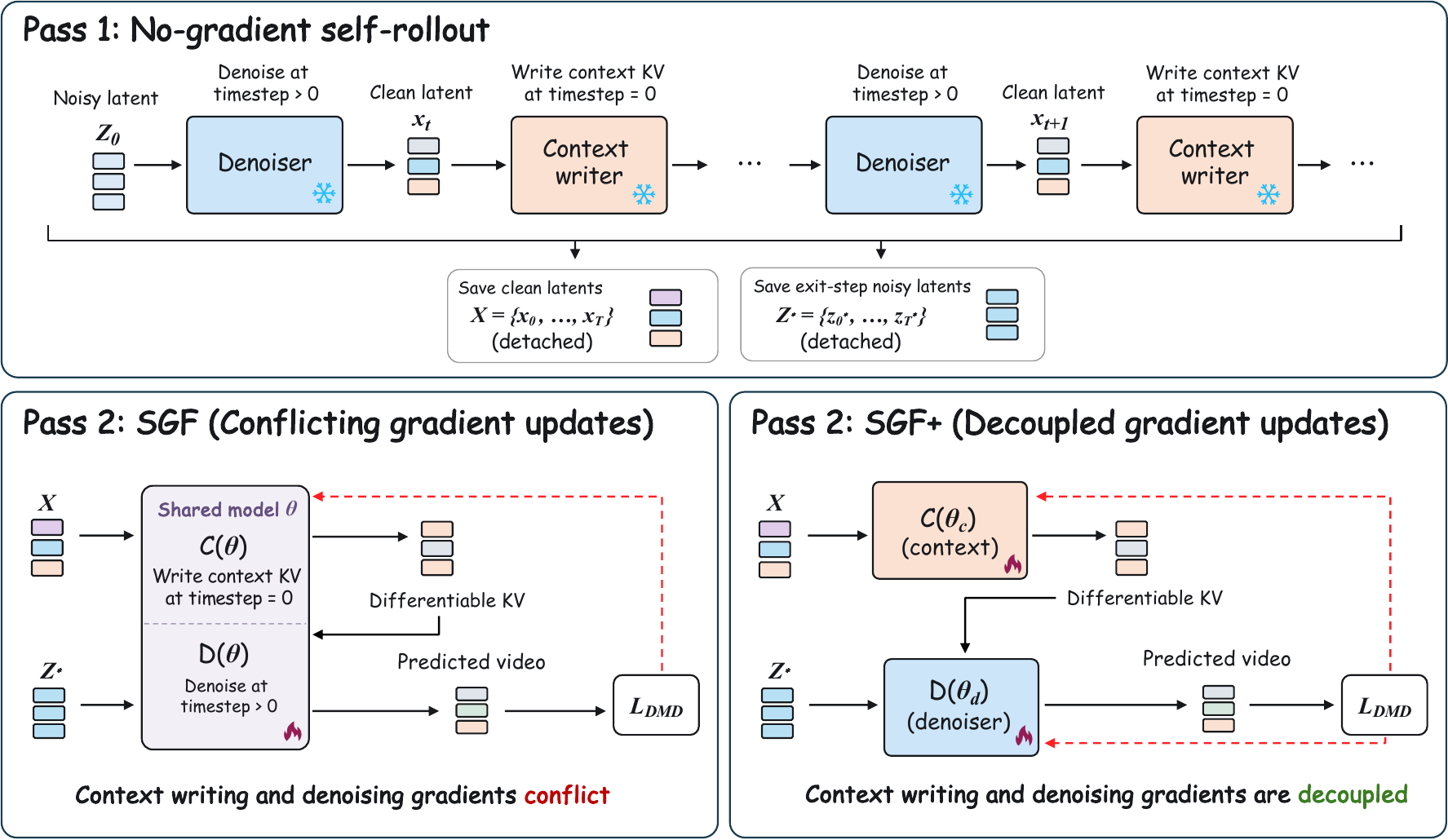}
  \caption{\textbf{Training pipeline of SGF and SGF+.} Pass 1 performs a no-gradient autoregressive rollout and records detached clean context latents $X$ together with noisy target latents $Z^{\star}$ at a sampled denoising exit. Pass 2 reconstructs the forward computation from the recorded latents, allowing future-generation losses to supervise context writing through differentiable context KV states. SGF accumulates context-writing and denoising gradients on shared parameters $\theta$, whereas SGF+ assigns separate parameters $\theta_c$ and $\theta_d$ to the two roles.}
  \label{fig:method-overview}
\end{figure*}

\section{Related Work}
\label{sec:related-work}

\subsection{Autoregressive Video Diffusion and World Models}

Autoregressive video diffusion combines iterative denoising with causal temporal prediction~\cite{sand2025magi,chen2025skyreels}. Causal video distillation and Self Forcing support efficient streaming inference through historical KV reuse~\cite{yin2025slowbidirectionalfastautoregressive,huang2025selfforcingbridgingtraintest}, while Rolling Forcing coordinates denoising across a moving temporal window~\cite{liu2025rollingforcingautoregressivelong}. Historical context also connects past outputs to future predictions in interactive world models. EchoWM, SolarWM, Zing-0.5, XPACE, and EditWorld adopt or adapt SGF's differentiable context reconstruction to supervise context writing through future generation losses, with applications spanning visual world simulation, interactive editing, and audio-video generation~\cite{huang2026solarwmopendatascalable,chen2026zing05playableworldsrealtime,wei2026xpacejointworldaction,zhang2026echowmopenenterableomnimodal,liao2026precise}. SGF+ complements these developments through role-specific parameterization, jointly optimizing context writing and denoising with separate parameters under the original generation objective, without auxiliary losses.

\subsection{Forcing Strategies for Autoregressive Video Generation}

Forcing methods differ in history construction, supervision, and gradient propagation. Teacher forcing uses clean ground-truth histories, whereas Diffusion Forcing independently samples noise levels across temporal tokens~\cite{chen2024diffusionforcingnexttokenprediction}. SF uses self-generated rollouts to reduce the mismatch between training and inference~\cite{huang2025selfforcingbridgingtraintest}, while Resampling Forcing combines self-resampled histories with a frame-level diffusion objective~\cite{guo2026endtoendtrainingautoregressivevideo}. Rolling Forcing coordinates denoising across adjacent frames at progressively different noise levels~\cite{liu2025rollingforcingautoregressivelong}, and Mask Forcing broadens rollout exploration through dual-noise masking under the existing distillation objective~\cite{zhao2026maskforcingimprovingautoregressive}.

Many subsequent methods build on SF's self-generated rollout framework to improve initialization or supervision. Causal Forcing uses an autoregressive teacher for ODE initialization before SF-style distribution matching~\cite{zhu2026causalforcingautoregressivediffusion}, while Causal Forcing++ replaces it with causal consistency distillation~\cite{zhao2026causalforcingpp}. Data-Forcing Distillation introduces data guidance into score-based updates to improve fidelity and diversity~\cite{chen2026dataforcingdistillationrestoringdiversity}. One-Forcing adds adversarial supervision from real videos~\cite{feng2026oneforcing}, Reward Forcing reweights DMD updates using motion rewards~\cite{lu2025rewardforcing}, and DuoMatching augments video-level distribution matching with frame-level marginal supervision from an image teacher~\cite{zhan2026duomatching}. Video-Mirai uses future-aware representation targets~\cite{yu2026videomiraiautoregressivevideodiffusion}, while Next Forcing introduces auxiliary multi-chunk prediction modules~\cite{xu2026forcingcausalworldmodeling}.

SGF extends SF's training framework by restoring gradients through historical KV construction, allowing future generation losses to supervise context writing~\cite{zhuang2026selfgradientforcingnative}. SGF+ revisits the underlying parameter-sharing design, jointly optimizing context writing and denoising with separate parameters under the original generation objective. This formulation retains short rollouts without auxiliary losses or additional video training data, and complements advances in history construction and supervision design.

\subsection{Long-Horizon Video Generation and Context Management}

Long-video methods improve temporal coverage, context access, and inference-time state management~\cite{henschel2025streamingt2v,song2025historyguided}. LongLive uses a short-video teacher to supervise successive segments of longer self-generated videos, combining this streaming long tuning with a bounded attention window, a frame sink, and KV recaching for prompt transitions~\cite{yang2025longliverealtimeinteractivelong}. Self-Forcing++ similarly supervises sampled segments of extended student rollouts, allowing the student's training horizon to exceed the teacher's supervision window~\cite{cui2025selfforcingminutescalehighqualityvideo}. Within a limited context budget, rolling caches support streaming generation~\cite{huang2025selfforcingbridgingtraintest}, attention sinks retain persistent anchors~\cite{liu2025rollingforcingautoregressivelong,yang2025longliverealtimeinteractivelong}, and history routing retrieves relevant earlier frames~\cite{guo2026endtoendtrainingautoregressivevideo}. Training-free methods further improve extrapolation: FreqForcing mitigates spectral drift through self-anchoring~\cite{li2026freqforcingautoregressivelongvideo}, while MemRoPE combines evolving memory tokens with online RoPE indexing~\cite{kim2026memropetrainingfreeinfinitevideo}.

SGF+ addresses how context writing and denoising are jointly learned within a limited temporal training window. Through role-specific parameterization, it achieves \emph{native long-horizon extrapolation} under the original generation objective, without extending the training horizon or applying long-video fine-tuning. This intervention is complementary in mechanism to context management, spectral correction, positional adaptation, and longer-horizon training.
\section{Method}
\label{sec:method}

Figure~\ref{fig:method-overview} presents the training pipelines of SGF and SGF+. Section~\ref{sec:method-sgf} reviews SGF's two-pass training procedure, which makes context writing differentiable. Section~\ref{sec:gradient-conflict} analyzes the gradient conflict caused by sharing parameters between context writing and denoising. Section~\ref{sec:sgf-plus} then introduces SGF+, which separates the two roles with role-specific parameters to achieve higher visual quality and stronger long-horizon consistency.

\subsection{Self Gradient Forcing Preliminaries}
\label{sec:method-sgf}

An autoregressive video diffusion model generates latent blocks sequentially. For block $i$, a denoiser reads the historical key-value (KV) cache and maps a noisy latent $z_i^t$ to a clean estimate $\tilde{x}_i$. A context-writing call then encodes this estimate into the memory used by later blocks:
\begin{equation}
  \mathrm{KV}_i
  = \mathcal{C}_{\theta}\!\left(\tilde{x}_i,t_{\mathrm{ctx}}\mid
  \mathrm{KV}_{<i}\right),
  \quad t_{\mathrm{ctx}}=0.
  \label{eq:cache-writing}
\end{equation}
The model therefore uses the same DiT parameters $\theta$ in two roles: denoising the current block and writing its clean representation for future predictions.

SGF trains this recurrent process with two passes. \textbf{Pass 1: no-gradient self-rollout.}\enspace The model performs the ordinary serial autoregressive rollout without retaining an autograd graph. It samples a denoising exit $t^{\star}$ and records the noisy exit latent $z_i^{t^{\star}}$ and its clean estimate $\tilde{x}_i$ at each autoregressive step. The clean estimates are processed at $t_{\mathrm{ctx}}=0$ to update the serial KV cache used by subsequent blocks. This produces the collections $Z^{\star}=\{z_i^{t^{\star}}\}_i$ and $X=\{\tilde{x}_i\}_i$, both detached from the sampling trajectory.

\textbf{Pass 2: parallel context-gradient reconstruction.}\enspace SGF reconstructs the sampled-exit computation in a single parallel forward pass under a causal reconstruction mask $\mathcal{M}_{\mathrm{rec}}$. The latents in $X$ remain stop-gradient inputs, but the context hidden states and their KV projections are recomputed with gradient tracking:
\begin{equation}
\begin{aligned}
  M_{\theta} &= \mathcal{C}_{\theta}
    \!\left(\operatorname{sg}(X),t_{\mathrm{ctx}}\mid\mathcal{M}_{\mathrm{rec}}\right), \\
  \hat{X}_{\mathrm{tar}} &= \mathcal{D}_{\theta}
    \!\left(Z^{\star},t^{\star}\mid M_{\theta},\mathcal{M}_{\mathrm{rec}}\right), \\
  \mathcal{L} &= \mathcal{L}_{\mathrm{DMD}}(\hat{X}_{\mathrm{tar}}).
\end{aligned}
  \label{eq:sgf-reconstruction}
\end{equation}
Consequently, the future-generation loss trains both how noisy target tokens read the history and how clean context tokens write it. The sampled latents and cache trajectory stay detached. Only the Pass-2 reconstruction is differentiable. This boundary restores the context-writing path without backpropagating through the full autoregressive rollout.

\subsection{Observations of Gradient Conflict}
\label{sec:gradient-conflict}

Restoring context-writing gradients improves temporal extrapolation but can slightly reduce visual quality in some cases (Appendix~\ref{sec:appendix-layers}). To understand this behavior, we isolate the gradient contributions of context writing and denoising within the same Pass-2 computation. For a shared linear weight $W$, let $H$ denote its input activations and $\Delta=\partial\mathcal{L}/\partial Y$ the corresponding output gradients. Partitioning the tokens by role gives
\begin{equation}
  g_{\mathrm{C}}=\Delta_{\mathrm{C}}^{\mathsf T}H_{\mathrm{C}},
  \quad
  g_{\mathrm{D}}=\Delta_{\mathrm{D}}^{\mathsf T}H_{\mathrm{D}},
  \quad
  g_W=g_{\mathrm{C}}+g_{\mathrm{D}},
  \label{eq:role-gradient-decomposition}
\end{equation}
where $g_{\mathrm{C}}$ and $g_{\mathrm{D}}$ are the contributions from context-writing and denoising tokens, respectively. Both contributions arise from the same objective rather than separate losses. We measure their directional alignment using cosine similarity:
\begin{equation}
  s(g_{\mathrm{C}},g_{\mathrm{D}})
  =\frac{\langle g_{\mathrm{C}},g_{\mathrm{D}}\rangle}
  {\lVert g_{\mathrm{C}}\rVert_2\lVert g_{\mathrm{D}}\rVert_2}.
  \label{eq:gradient-cosine}
\end{equation}
Large angles indicate directional disagreement between the two gradients. A pair is conflicting when $s<0$, equivalently when the angle between the gradients exceeds $90^{\circ}$.

Figure~\ref{fig:gradient-conflict} visualizes the gradient relationships in SGF. \textit{(a) Gradient distributions.} Angular-distance t-SNE shows clearly separated distributions of context-writing and denoising gradients in both Attention and FFN. \textit{(b) Mean gradient directions.} The mean gradient directions of the two roles exhibit large angular separation, even exceeding $90^{\circ}$ in both module families, indicating overall negative alignment. \textit{(c) Gradient conflict across timesteps.} All 512 paired cosine similarities are negative in both module families, indicating persistent conflict across the measured prompts and denoising timesteps. Appendix~\ref{sec:appendix-layers} provides the complete layer-wise gradient distributions and additional evidence of gradient conflict in the larger, 22B-parameter LTX-2.5 model~\cite{lightricks2026ltx25}.

This directional relationship directly affects the gradient received by the shared parameters. Since $g_W=g_{\mathrm{C}}+g_{\mathrm{D}}$, conflicting gradient contributions partially cancel on shared parameters. This motivates assigning independent parameters to context writing and denoising, separating their updates while preserving their forward interaction. Appendix~\ref{sec:role-separation-theory} explains the advantage of role separation under gradient conflict from a local optimization perspective.

\subsection{Self Gradient Forcing Plus}
\label{sec:sgf-plus}

Self Gradient Forcing Plus assigns independent parameters to context writing and denoising, forming a context writer $\mathcal{C}_{\theta_c}$ and a denoiser $\mathcal{D}_{\theta_d}$. Both are initialized from the same autoregressive diffusion model. Pass 1 retains the no-gradient rollout in Figure~\ref{fig:method-overview}: $\mathcal{D}_{\theta_d}$ generates each block, and $\mathcal{C}_{\theta_c}$ writes its clean estimate into the cache. Pass 2 becomes
\begin{equation}
\begin{aligned}
  M_{\theta_c} &= \mathcal{C}_{\theta_c}
    \!\left(\operatorname{sg}(X),t_{\mathrm{ctx}}\mid\mathcal{M}_{\mathrm{rec}}\right), \\
  \hat{X}_{\mathrm{tar}} &= \mathcal{D}_{\theta_d}
    \!\left(Z^{\star},t^{\star}\mid M_{\theta_c},\mathcal{M}_{\mathrm{rec}}\right), \\
  \mathcal{L} &= \mathcal{L}_{\mathrm{DMD}}(\hat{X}_{\mathrm{tar}}).
\end{aligned}
  \label{eq:sgf-plus}
\end{equation}
The future-generation objective reaches $\theta_c$ through the differentiable KV state $M_{\theta_c}$ and $\theta_d$ through target denoising. The context writer and denoiser remain forward-coupled and jointly trained, while their updates no longer accumulate on the same weights. SGF+ reroutes SGF's existing gradient paths without an auxiliary context loss or gradient projection.

We adopt parameter separation as the simplest and most direct way to prevent conflicting role updates. Beyond parameter separation, other approaches such as model merging~\cite{wei2025modelingmultitask,ilharco2023taskarithmetic,yadav2023ties,yu2024dare} offer insights into coordinating the optimization of different roles. We leave this for future work.

\endgroup

\setcounter{dbltopnumber}{2}
\begin{figure*}[t]
  \centering
  \includegraphics[width=\textwidth]{figures/qualitative_chunkwise/chunkwise_240s_01.pdf}
  \vspace{2pt}
  \includegraphics[width=\textwidth]{figures/qualitative_chunkwise/chunkwise_240s_02.pdf}
  \vspace{2pt}
  \includegraphics[width=\textwidth]{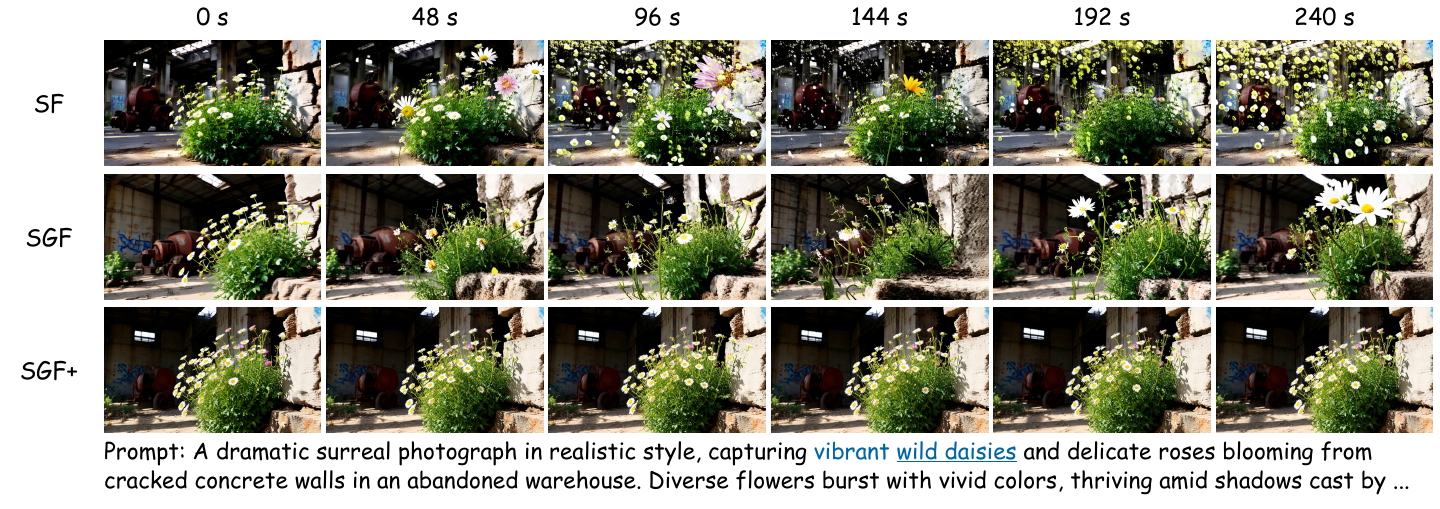}
  \caption{\textbf{Qualitative comparisons.} The top two cases show chunkwise generation, and the bottom case shows framewise generation. Additional qualitative comparisons covering a wider variety of scenes are provided in Appendix~\ref{sec:appendix-qualitative}.}
  \label{fig:qualitative-chunkwise-main}
\end{figure*}

\begin{table*}[t]
\centering
\caption{\textbf{Framewise and chunkwise long-horizon metrics at 60s and 240s.} The 60s setting uses VBench-Long prompts, and the 240s setting uses MovieGen-128 prompts. All methods share the same initialization, prompt set, and random seed within each setting.}
\label{tab:vbench-chunkwise}
\begingroup
\small
\renewcommand{\arraystretch}{1.13}
\begin{tabular}{l*{7}{c}}
\toprule
\textbf{Method} & \mbox{\fontsize{8}{10}\selectfont\bfseries\boldmath Subject} & \mbox{\fontsize{8}{10}\selectfont\bfseries\boldmath Background} & \mbox{\fontsize{8}{10}\selectfont\bfseries\boldmath Flickering} & \mbox{\fontsize{8}{10}\selectfont\bfseries\boldmath Motion} & \mbox{\fontsize{8}{10}\selectfont\bfseries\boldmath Dynamics} & \mbox{\fontsize{8}{10}\selectfont\bfseries\boldmath Aesthetics} & \mbox{\fontsize{8}{10}\selectfont\bfseries\boldmath Imaging} \\
\midrule
\multicolumn{8}{c}{\textbf{\textit{60s Generation Horizon \textperiodcentered{} Chunkwise}}} \\
\midrule
Self Forcing & 94.97 & 94.85 & 98.88 & 96.69 & \textbf{93.46} & 58.09 & 68.12 \\
Self Gradient Forcing & 98.17 & 97.12 & 98.94 & 98.53 & 64.16 & 65.38 & 71.06 \\
\rowcolor{sgfgray}
Self Gradient Forcing Plus & \textbf{98.48} & \textbf{97.41} & \textbf{99.35} & \textbf{98.69} & 63.93 & \textbf{66.63} & \textbf{71.51} \\
\midrule
\multicolumn{8}{c}{\textbf{\textit{60s Generation Horizon \textperiodcentered{} Framewise}}} \\
\midrule
Self Forcing & 97.62 & 97.42 & 99.08 & 98.55 & 63.89 & 64.80 & 70.50 \\
Self Gradient Forcing & 98.26 & 96.99 & 99.08 & 98.17 & \textbf{71.19} & 64.95 & \textbf{71.52} \\
\rowcolor{sgfgray}
Self Gradient Forcing Plus & \textbf{99.08} & \textbf{98.06} & \textbf{99.43} & \textbf{98.90} & 55.38 & \textbf{66.29} & 71.11 \\
\midrule
\multicolumn{8}{c}{\textbf{\textit{240s Generation Horizon \textperiodcentered{} Chunkwise}}} \\
\midrule
Self Forcing & 94.94 & 95.21 & 93.54 & 96.75 & \textbf{93.90} & 55.12 & 68.38 \\
Self Gradient Forcing & 97.72 & 96.90 & 97.03 & 98.50 & 56.97 & 62.54 & 71.11 \\
\rowcolor{sgfgray}
Self Gradient Forcing Plus & \textbf{98.21} & \textbf{97.31} & \textbf{97.57} & \textbf{98.63} & 56.92 & \textbf{64.74} & \textbf{71.39} \\
\midrule
\multicolumn{8}{c}{\textbf{\textit{240s Generation Horizon \textperiodcentered{} Framewise}}} \\
\midrule
Self Forcing & 96.36 & 96.42 & 96.73 & 98.01 & \textbf{67.04} & 61.36 & 70.71 \\
Self Gradient Forcing & 97.07 & 96.81 & 96.85 & 98.24 & 65.68 & 61.89 & 70.89 \\
\rowcolor{sgfgray}
Self Gradient Forcing Plus & \textbf{98.74} & \textbf{97.81} & \textbf{98.24} & \textbf{98.94} & 53.96 & \textbf{63.95} & \textbf{70.90} \\
\bottomrule
\end{tabular}
\endgroup
\end{table*}

\begin{table*}[t]
\centering
\caption{\textbf{Ablation of modules for parameter separation.} Attention only and FFN only apply parameter separation only to Attention and FFN modules, respectively. The full model applies role-specific separation to all model parameters.}
\label{tab:ablation}
\begingroup
\small
\renewcommand{\arraystretch}{1.16}
\begin{tabular}{l*{7}{c}}
\toprule
\textbf{Method} & \mbox{\fontsize{8}{10}\selectfont\bfseries\boldmath Subject} & \mbox{\fontsize{8}{10}\selectfont\bfseries\boldmath Background} & \mbox{\fontsize{8}{10}\selectfont\bfseries\boldmath Flickering} & \mbox{\fontsize{8}{10}\selectfont\bfseries\boldmath Motion} & \mbox{\fontsize{8}{10}\selectfont\bfseries\boldmath Dynamics} & \mbox{\fontsize{8}{10}\selectfont\bfseries\boldmath Aesthetics} & \mbox{\fontsize{8}{10}\selectfont\bfseries\boldmath Imaging} \\
\midrule
Attention only & 95.37 & 95.43 & 99.30 & 97.40 & \textbf{93.41} & 58.49 & 68.18 \\
FFN only & 97.43 & 96.48 & 98.33 & 98.27 & 65.37 & 63.49 & 70.41 \\
\rowcolor{sgfgray}
Full model (Ours) & \textbf{98.48} & \textbf{97.41} & \textbf{99.35} & \textbf{98.69} & 63.93 & \textbf{66.63} & \textbf{71.51} \\
\bottomrule
\end{tabular}
\endgroup
\end{table*}

\begin{table*}[t]
\centering
\caption{\textbf{Training and inference efficiency.} Compared with SGF, SGF+ modestly increases training time per step, with a small increase in inference memory and nearly unchanged inference latency.}
\label{tab:efficiency}
\begingroup
\small
\renewcommand{\arraystretch}{1.16}
\begin{tabular}{l*{5}{c}}
\toprule
\textbf{Method} & \mbox{\fontsize{8}{10}\selectfont\bfseries\boldmath Train Peak Memory} & \mbox{\fontsize{8}{10}\selectfont\bfseries\boldmath Train Stable Memory} & \mbox{\fontsize{8}{10}\selectfont\bfseries\boldmath Train Time / step} & \mbox{\fontsize{8}{10}\selectfont\bfseries\boldmath Infer Memory} & \mbox{\fontsize{8}{10}\selectfont\bfseries\boldmath Infer Time / 81 frames} \\
\midrule
Self Forcing & 86.36~GB & 86.19~GB & 10.02~s & 24.85~GB & 4.963~s \\
Self Gradient Forcing & 97.83~GB & 70.06~GB & 11.79~s & 24.85~GB & 4.962~s \\
\rowcolor{sgfgray}
Self Gradient Forcing Plus & 98.15~GB & 73.60~GB & 12.76~s & 27.96~GB & 4.969~s \\
\bottomrule
\end{tabular}
\endgroup
\end{table*}

\section{Experiments}
\label{sec:experiments}

\subsection{Experimental Setup}
\label{sec:experimental-setup}

We use Teacher Forcing (TF) initialization with weights released by Causal Forcing~\cite{zhu2026causalforcingautoregressivediffusion} and train on 5s video windows using filtered and expanded prompts from VidProM~\cite{wang2024vidprom}. Following SGF~\cite{zhuang2026selfgradientforcingnative}, we compare SF, SGF, and SGF+ in framewise and chunkwise generation at 5s, 60s, and 240s. The 60s and 240s evaluations therefore test native extrapolation beyond the training horizon. All models use a causal video diffusion student based on Wan2.1-T2V-1.3B~\cite{wan2025}, a frozen Wan2.1-T2V-14B teacher, and a separately trained fake-score network initialized from Wan2.1-T2V-1.3B. SGF+ can implement role-specific parameter separation using either two separate LoRA adapters or two full parameter sets. In this work, we use the latter.

\begingroup
\setlength{\parskip}{3pt}
The 5s evaluation follows standard VBench~\cite{huang2023vbench}, with full results in Appendix~\ref{sec:appendix-short}. The main comparison evaluates 60s and 240s generation using VBench-Long prompts~\cite{huang2025vbenchpp} and 128 MovieGen prompts~\cite{polyak2024movie}, respectively. For both horizons, we report Subject Consistency, Background Consistency, Temporal Flickering, Motion Smoothness, Dynamic Degree, Aesthetic Quality, and Imaging Quality. Scores are multiplied by 100, and higher is better. Further implementation details are provided in Appendix~\ref{sec:appendix-implementation}.\par
\endgroup

\subsection{Native Long-Horizon Extrapolation}
\label{sec:long-horizon-results}

\noindent\textbf{Qualitative comparisons.}\enspace
Figure~\ref{fig:qualitative-chunkwise-main} compares SF, SGF, and SGF+ over 240s rollouts in chunkwise and framewise generation. In the gold-rush town scene, SGF+ preserves a more coherent rendering of buildings and miners, while SF exhibits background distortions and SGF develops conspicuous green streak artifacts. In the salt-desert scene, SGF+ better maintains the explorer's appearance and desert setting, with fewer helmet distortions and sky artifacts than the baselines. In the warehouse-flower scene, SF produces floating flower-like artifacts and SGF exhibits pronounced changes in flower shape and scale, while SGF+ better preserves the flower cluster and surrounding scene. These examples illustrate improvements in both visual quality and long-horizon consistency. Additional qualitative comparisons in both generation settings are provided in Appendix~\ref{sec:appendix-qualitative}. Trained only on 5s video windows, SGF+ supports continuous generation for up to 24 hours, as shown in Figure~\ref{fig:teaser} and Appendix~\ref{sec:appendix-24h}.

\noindent\textbf{Quantitative comparisons.}\enspace
Table~\ref{tab:vbench-chunkwise} compares the three methods at generation horizons of $12\times$ and $48\times$ the training window. Across both generation granularities and both long horizons, SGF+ outperforms SF and SGF on most quality and consistency metrics, with gains in subject consistency, background consistency, temporal flickering, motion smoothness, aesthetic quality, and imaging quality. These results indicate that separating the parameters of context writing and denoising to avoid conflicting gradients on shared weights helps preserve both visual quality and long-horizon consistency during extended generation.
Dynamic Degree is the main exception, with SF or SGF achieving higher scores. This does not necessarily indicate better motion quality: scene jumps, object deformation, subject disappearance, or additional people appearing (as shown in Figures~\ref{fig:teaser} and~\ref{fig:qualitative-chunkwise-main}) can produce large but incoherent apparent motion, inflating Dynamic Degree. In contrast, SGF+ maintains more stable visuals and more coherent motion, so its lower Dynamic Degree does not imply worse generation quality.

\subsection{Ablation Study}
\label{sec:role-ablation}

The comparison between SGF and SGF+ in Section~\ref{sec:long-horizon-results} demonstrates the benefits of parameter separation. We further ablate parameter separation in Attention and FFN. Table~\ref{tab:ablation} considers these two main Transformer module families. Both partial-separation variants underperform the full model on quality and consistency metrics. Thus, an MoE-style design~\cite{fedus2022switch,lepikhin2021gshard,du2022glam,jiang2024mixtral} that introduces role experts only in FFN cannot fully eliminate the conflict, as conflicting gradients also arise in Attention. Attention-only separation achieves the highest Dynamic Degree but scores substantially lower than FFN-only separation and the full model on most other metrics, again showing that a higher Dynamic Degree does not necessarily indicate better generation quality.

\subsection{Discussion}
\label{sec:experiment-discussion}

\par\smallskip
\noindent\textbf{Efficiency.}\enspace
Table~\ref{tab:efficiency} shows that SGF+ doubles generator parameters from 1.4B to 2.8B while preserving SGF's denoising and context-writing call schedule. Each token uses only its role-specific parameters within a forward pass, so computation does not double. For the evaluated 81-frame workload, inference time remains nearly unchanged at 4.969 seconds versus 4.962 seconds for SGF, while memory increases from 24.85 to 27.96 GB. Training peak and stable memory increase from 97.83 to 98.15 GB and from 70.06 to 73.60 GB, respectively; time per step rises from 11.79 to 12.76 seconds, an increase of approximately 8.2\%. Thus, the additional parameter storage incurs only modest memory overhead, alongside an 8.2\% increase in training time per step and nearly unchanged inference latency under the evaluated workload.

\par\smallskip
\noindent\textbf{Role-specific gradient conflict at TF initialization.}\enspace
Role-specific gradient differences and conflict are already observable at teacher-forcing (TF) initialization, before SGF training. Figure~\ref{fig:adjacent-gradient} shows a sharp directional change between diffusion timestep indices 0 and 1, at the transition from clean-context writing to noisy denoising. Appendix~\ref{sec:appendix-tf-gradients} further shows mean role-gradient angles exceeding $90^\circ$ in both Attention and FFN, with the vast majority of paired cosine similarities being negative. These findings motivate investigating role-specific parameterization from the TF stage and its potential extension to world action models (WAMs)~\cite{ye2026dreamzero,li2026lingbotva,wang2026repwam,peng2026fact} trained with teacher forcing, which use context KV representations for action prediction.

\begin{figure}[t]
\centering
\includegraphics[width=\linewidth]{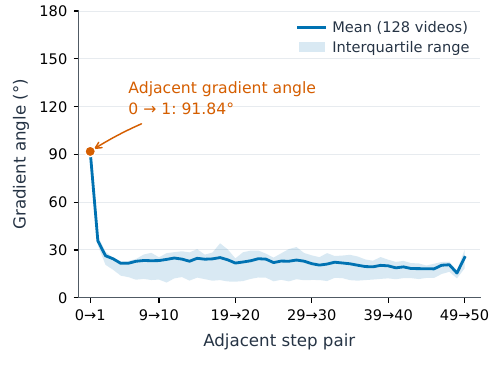}
\captionsetup{skip=8pt}
\caption{\textbf{Adjacent-timestep gradient directions at TF initialization.} We measure gradient directions on 128 video samples over a 50-step diffusion schedule. The curve shows the mean angle between adjacent timesteps. The sharp $0\!\rightarrow\!1$ transition reveals distinct gradient directions for clean-context writing and denoising. Additional visualizations of context-writing and denoising gradients at TF initialization are provided in Appendix~\ref{sec:appendix-tf-gradients}.}
\label{fig:adjacent-gradient}
\end{figure}

\par\smallskip
\noindent\textbf{Relation to autoregressive language models.}\enspace
Standard autoregressive LLMs~\cite{brown2020gpt3,touvron2023llama,touvron2023llama2,grattafiori2024llama3} encode discrete text tokens, write KV representations, and predict subsequent tokens in one forward pass; historical representations also receive gradients from future token losses during training. Autoregressive video diffusion alternates between clean-context encoding and noisy-latent denoising. These distinct input conditions and computational roles may help explain the observed gradient differences. Our results highlight role-specific parameterization as an effective design principle for autoregressive video training: context writing and denoising remain jointly optimized under the original generation objective, without auxiliary losses, additional video training data, or long-video fine-tuning.

\section{Conclusion}
\label{sec:conclusion}

We identified an optimization conflict between context writing and denoising in shared-parameter Self Gradient Forcing. SGF+ resolves this conflict through role-specific parameter separation, improving visual quality and long-horizon consistency over SF and SGF in both framewise and chunkwise generation. It supports generation for up to 24 hours from 5s training rollouts. This role distinction is already evident at TF initialization, motivating role-aware parameterization throughout autoregressive video training.

{\small
\bibliographystyle{ieeenat_fullname}
\bibliography{references}
}
\clearpage
\appendix
\onecolumn
\begingroup
\begin{minipage}{\textwidth}
\centering
{\Large\bfseries\thetitle\par}
\vspace{0.5em}
{\Large Supplementary Material\par}
\vspace{1em}
\raggedright
We provide additional information in the supplementary material, as outlined below:
\begingroup
\hypersetup{linkcolor=red}
\begin{itemize}[leftmargin=1.5em,itemsep=2pt,parsep=0pt,topsep=3pt]
  \item \textbf{Sec.}~\ref{sec:appendix-short}: Additional Short-Video Results.
  \item \textbf{Sec.}~\ref{sec:appendix-implementation}: Additional Implementation Details.
  \item \textbf{Sec.}~\ref{sec:appendix-tf-gradients}: Gradient Role Separation at TF Initialization.
  \item \textbf{Sec.}~\ref{sec:appendix-24h}: 24-Hour Generation Results.
  \item \textbf{Sec.}~\ref{sec:appendix-layers}: Additional Analysis of Gradient Conflict.
  \item \textbf{Sec.}~\ref{sec:appendix-qualitative}: Additional Qualitative Comparisons.
  \item \textbf{Sec.}~\ref{sec:role-separation-theory}: Local Optimization Analysis of Role Separation.
  \item \textbf{Sec.}~\ref{sec:future-work}: Future Work.
\end{itemize}
\endgroup
\section{Additional Short-Video Results}
\label{sec:appendix-short}
\begingroup
\renewenvironment{table*}[1][]{\begingroup\captionsetup{type=table,position=top}}{\endgroup}
\begin{table*}[t]
\centering
\caption{\textbf{Framewise and chunkwise VBench metrics at 5s.} SF, SGF, and SGF+ are evaluated within the training horizon using standard VBench. Bold marks the best value within each generation mode.}
\label{tab:vbench-5s}
\begingroup
\small
\renewcommand{\arraystretch}{1.13}
\begin{tabular}{l*{7}{c}}
\toprule
\textbf{Method} & \mbox{\fontsize{8}{10}\selectfont\bfseries\boldmath Subject} & \mbox{\fontsize{8}{10}\selectfont\bfseries\boldmath Background} & \mbox{\fontsize{8}{10}\selectfont\bfseries\boldmath Flickering} & \mbox{\fontsize{8}{10}\selectfont\bfseries\boldmath Motion} & \mbox{\fontsize{8}{10}\selectfont\bfseries\boldmath Dynamics} & \mbox{\fontsize{8}{10}\selectfont\bfseries\boldmath Aesthetics} & \mbox{\fontsize{8}{10}\selectfont\bfseries\boldmath Imaging} \\
\midrule
\multicolumn{8}{c}{\textbf{\textit{Chunkwise}}} \\
\midrule
Self Forcing & 93.74 & 94.74 & 97.83 & 97.58 & \textbf{89.72} & 66.35 & 69.46 \\
Self Gradient Forcing & 97.01 & 96.42 & 98.80 & 98.50 & 65.28 & 68.04 & 70.56 \\
\rowcolor{sgfgray}
Self Gradient Forcing Plus & \textbf{97.24} & \textbf{96.81} & \textbf{99.30} & \textbf{98.64} & 62.78 & \textbf{68.45} & \textbf{70.64} \\
\midrule
\multicolumn{8}{c}{\textbf{\textit{Framewise}}} \\
\midrule
Self Forcing & 96.06 & 96.15 & 98.81 & 98.54 & 63.89 & 67.11 & 69.83 \\
Self Gradient Forcing & 96.69 & 96.16 & 98.97 & 98.24 & \textbf{64.72} & 66.95 & \textbf{71.41} \\
\rowcolor{sgfgray}
Self Gradient Forcing Plus & \textbf{98.16} & \textbf{97.20} & \textbf{99.47} & \textbf{98.96} & 54.61 & \textbf{67.22} & 70.88 \\
\bottomrule
\end{tabular}

\vspace{5pt}
\begin{tabular}{l*{8}{c}}
\toprule
\textbf{Method} & \mbox{\fontsize{8}{10}\selectfont\bfseries\boldmath Object} & \mbox{\fontsize{8}{10}\selectfont\bfseries\boldmath Multiple} & \mbox{\fontsize{8}{10}\selectfont\bfseries\boldmath Action} & \mbox{\fontsize{8}{10}\selectfont\bfseries\boldmath Color} & \mbox{\fontsize{8}{10}\selectfont\bfseries\boldmath Spatial} & \mbox{\fontsize{8}{10}\selectfont\bfseries\boldmath Scene} & \mbox{\fontsize{8}{10}\selectfont\bfseries\boldmath Appearance} & \mbox{\fontsize{8}{10}\selectfont\bfseries\boldmath Temporal} \\
\midrule
\multicolumn{9}{c}{\textbf{\textit{Chunkwise}}} \\
\midrule
Self Forcing & 95.62 & 85.35 & 96.00 & 86.89 & 81.43 & 53.90 & 20.20 & 24.18 \\
Self Gradient Forcing & 96.20 & 86.60 & 96.00 & 87.41 & 78.96 & 54.06 & 20.34 & 24.16 \\
\rowcolor{sgfgray}
Self Gradient Forcing Plus & \textbf{96.57} & \textbf{88.95} & \textbf{97.40} & \textbf{88.90} & \textbf{82.35} & \textbf{54.68} & \textbf{20.40} & \textbf{24.25} \\
\midrule
\multicolumn{9}{c}{\textbf{\textit{Framewise}}} \\
\midrule
Self Forcing & 94.87 & 85.91 & 95.80 & 84.33 & 75.62 & 54.17 & 20.13 & 23.76 \\
Self Gradient Forcing & 94.46 & 85.50 & \textbf{96.00} & \textbf{87.87} & 73.69 & 54.88 & 20.37 & 23.87 \\
\rowcolor{sgfgray}
Self Gradient Forcing Plus & \textbf{96.04} & \textbf{87.36} & 94.40 & 87.12 & \textbf{78.23} & \textbf{55.67} & \textbf{20.49} & \textbf{23.92} \\
\bottomrule
\end{tabular}
\par
\endgroup
\end{table*}

\endgroup
\par
\end{minipage}
\endgroup
\section{Additional Implementation Details}
\label{sec:appendix-implementation}
Within each generation mode, the compared methods share the same sink-plus-FIFO context policy. Framewise generation uses 4 sink latents, 16 recent latents in the FIFO cache, and 1 current latent, giving a total window of 21. Chunkwise generation uses 3 sink latents, 6 recent latents, and a current chunk of 3 latents, giving a total window of 12.

We optimize both the generator and the critic using AdamW with $\beta_1=0$ and $\beta_2=0.999$. The learning rates are $2\times10^{-6}$ for the generator and $4\times10^{-7}$ for the critic. We perform one generator update for every five critic updates and use four denoising steps for generation. We train on 64 GPUs with a per-GPU batch size of 1 and no gradient accumulation. Chunkwise and framewise SGF+ are trained for 700 and 1,900 steps, respectively.

\section{Gradient Role Separation at TF Initialization}
  \label{sec:appendix-tf-gradients}
  \noindent\begin{minipage}{\textwidth}
  \centering
  \captionsetup{type=figure}
  \includegraphics[width=\textwidth]{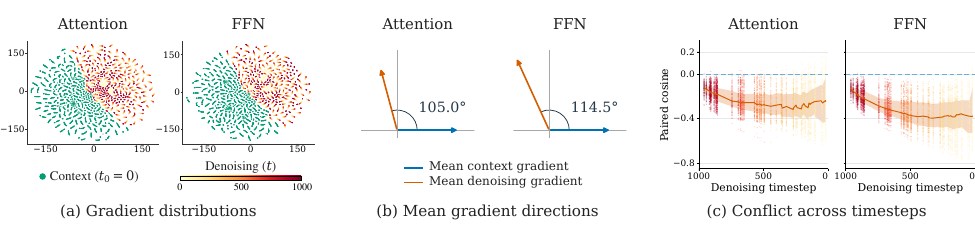}
  \caption{\textbf{Visualization of context--denoising gradients at TF initialization.} We use 128 video samples and 50 randomly sampled denoising timesteps, with context fixed at $t=0$. (a) Angular-distance t-SNE visualizes 6,400 context-writing and 6,400 denoising gradient samples per module family using coordinate-sampled gradient representations. Color indicates the denoising timestep. (b) Mean gradient directions and relative lengths are estimated from the sampled coordinates, with the context mean length normalized to one in each family. (c) Paired cosine similarities are computed from the full selected weight gradients. Points show video--timestep pairs, while lines and bands show the median and interquartile range across videos.}
  \label{fig:tf-init-role-gradients}
  \end{minipage}
  \par\medskip
  \raggedright
  Figure~\ref{fig:tf-init-role-gradients} provides additional visualizations of the gradient relationships between context writing and denoising at TF initialization. The two roles exhibit distinct gradient distributions in Attention and FFN, with mean directions forming angles of $105.0^{\circ}$ and $114.5^{\circ}$, respectively, both exceeding $90^{\circ}$. The vast majority of paired cosine similarities are negative, indicating that role separation and gradient conflict are already present at TF initialization and supporting separation of the two roles from the TF stage.
\section{24-Hour Generation Results}
  \label{sec:appendix-24h}
  Figure~\ref{fig:qualitative-24h} presents four examples of continuous 24-hour generation by SGF+ trained on 5s video windows, without long-video fine-tuning.
  \par\medskip
  \noindent\begin{minipage}{\textwidth}
  \centering
  \captionsetup{type=figure}
  \includegraphics[width=\textwidth,trim=0 22bp 0 0,clip]{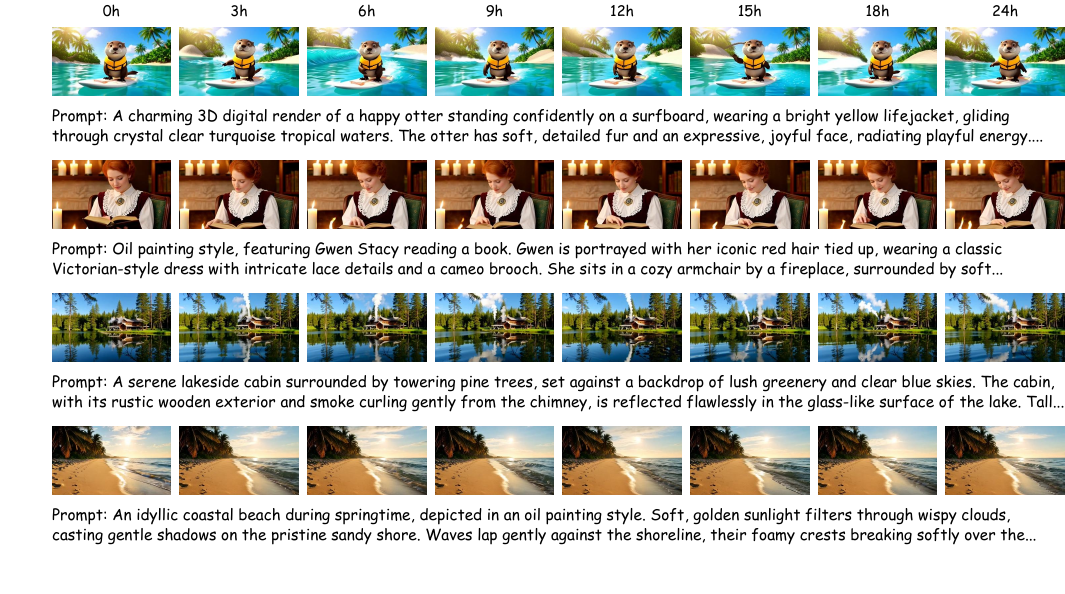}
  \caption{\textbf{24-hour generation results of SGF+.} SGF+ is trained on 5s video windows without long-video fine-tuning.}
  \label{fig:qualitative-24h}
  \end{minipage}
  \raggedright
  \par
\clearpage
\section{Additional Analysis of Gradient Conflict}
\label{sec:appendix-layers}
\noindent\begin{minipage}{\textwidth}
  \centering
  \captionsetup{type=figure,skip=4pt}
  \includegraphics[width=\textwidth]{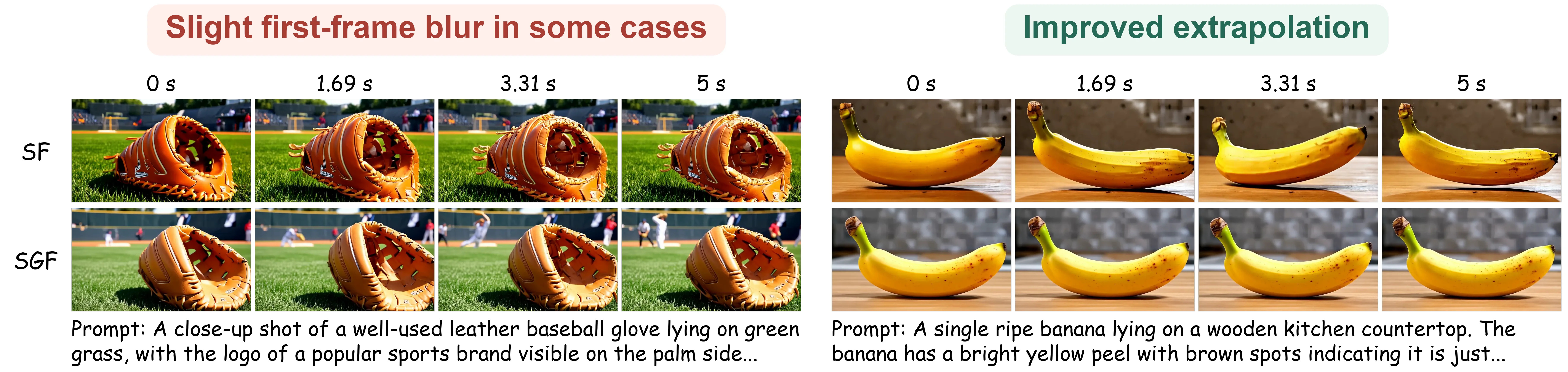}
  \caption{\textbf{Qualitative results of gradient conflict.} Restoring context-writing gradients improves temporal extrapolation but also constrains denoising optimization, leading to a slight reduction in visual quality in some cases. In the left example, SF produces sharper details in the baseball glove and grass. In the right example, SGF maintains better temporal consistency with less error accumulation.}
  \label{fig:sf-sgf-quality-consistency}
\end{minipage}
\par\vspace{10pt}
\noindent\begin{minipage}{\textwidth}
  \centering
  \captionsetup{type=figure,skip=4pt}
  \includegraphics[width=\textwidth]{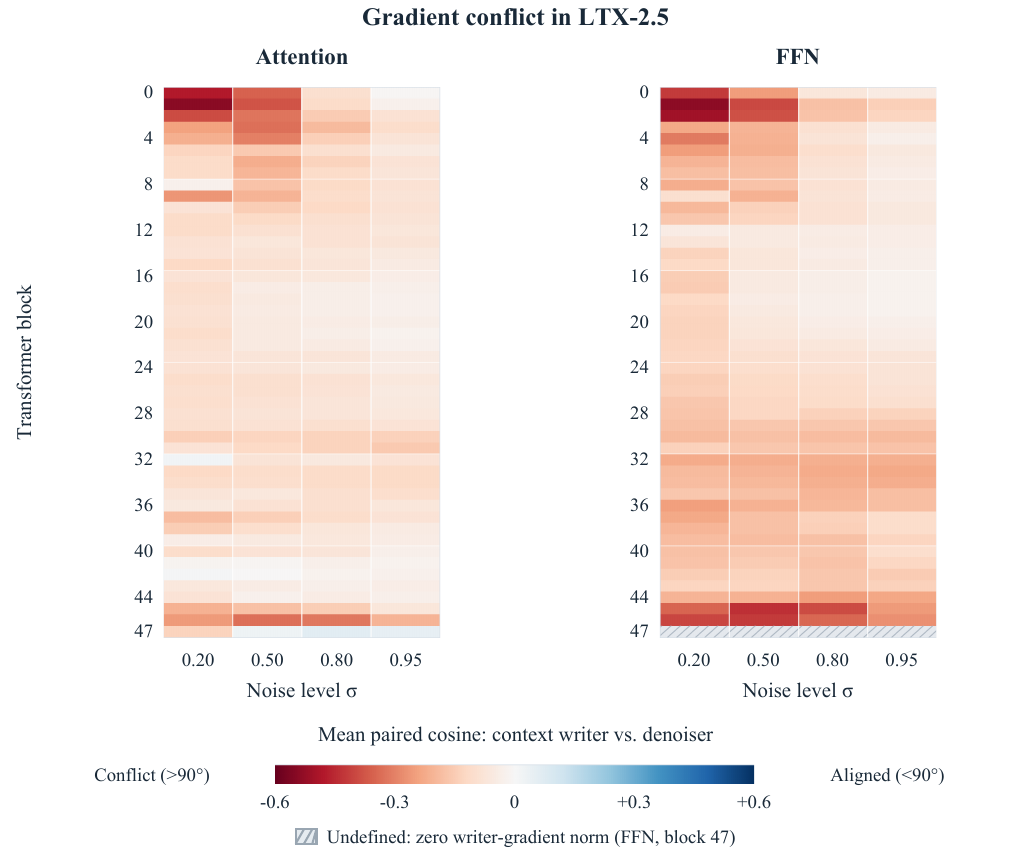}
  \caption{\textbf{Gradient conflict in the larger-scale, 22B-parameter LTX-2.5 model.} These measurements further confirm the presence of gradient conflict between context writing and denoising in a larger model. Each cell shows the mean paired gradient cosine similarity at a given layer and noise level. Red indicates negative cosine similarity, while blue indicates positive cosine similarity. In FFN block 47, the context-writing gradient has zero norm, making the cosine similarity undefined, as indicated by hatching.}
  \label{fig:ltx25-gradient-conflict}
\end{minipage}
\clearpage
\noindent\begin{minipage}{\textwidth}
  \captionsetup{type=figure}
  \setlength{\abovecaptionskip}{4pt}
  \centering
  \includegraphics[width=\textwidth,trim=0 0 0 5bp,clip]{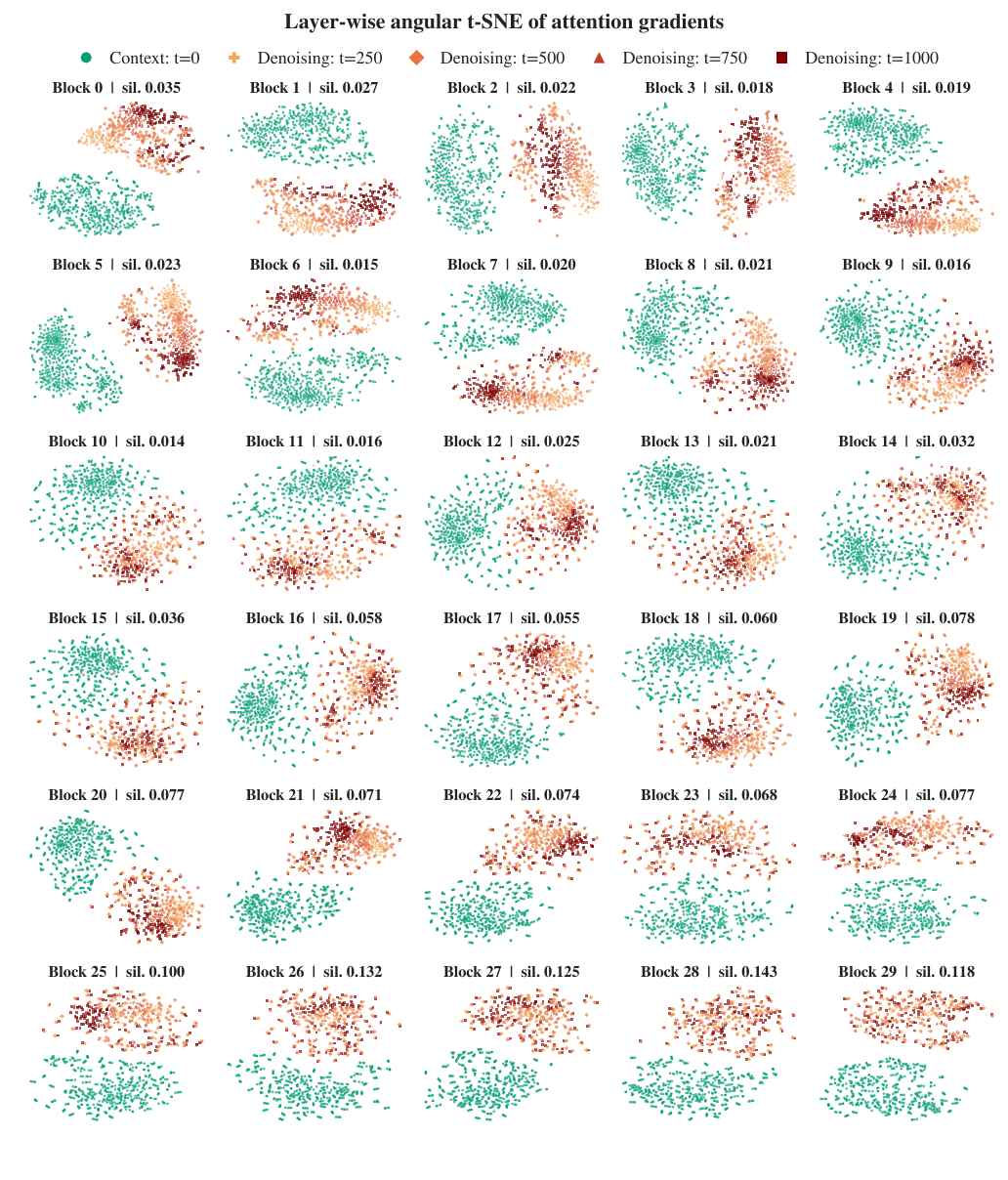}
  \caption{\textbf{Layer-wise gradient distributions in Attention.} We analyze SGF gradients over 128 prompts and 4 denoising timesteps, concatenating the Q/K/V/O weight gradients within each block. Angular-distance t-SNE uses gradient sketches with perplexity 20. \textit{sil.} is the role silhouette score computed from angular distances before projection. Higher scores indicate stronger role separation.}
  \label{fig:layer-attention-tsne}
\end{minipage}
\begin{figure}[!htp]
  \setlength{\abovecaptionskip}{4pt}
  \centering
  \includegraphics[width=\textwidth]{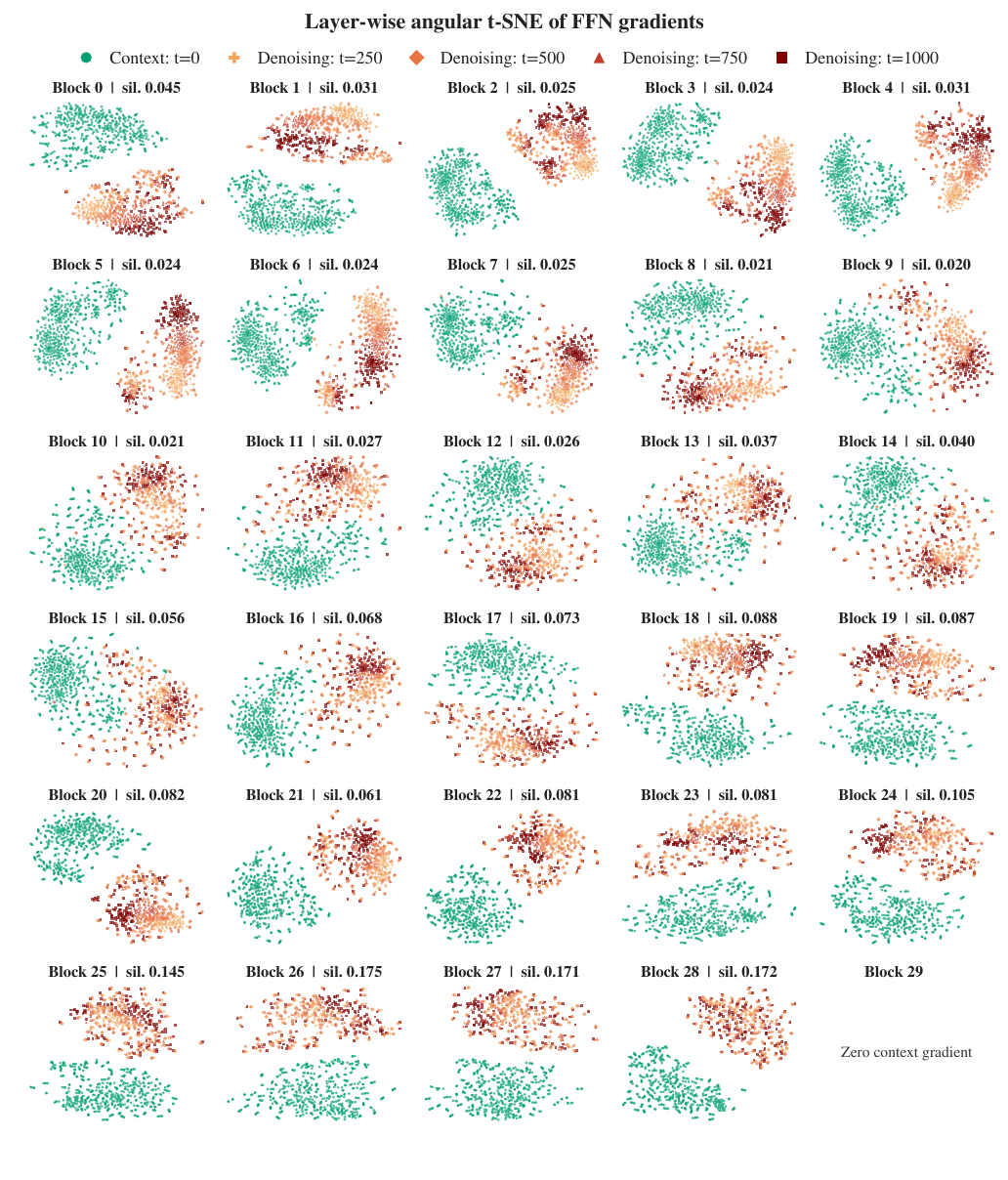}
  \caption{\textbf{Layer-wise gradient distributions in FFN.} We concatenate the up/down weight gradients within each block and use the settings in Figure~\ref{fig:layer-attention-tsne}. Block 29 has zero context-writing gradients for all 512 prompt--timestep pairs, so its angular distances are undefined and no t-SNE is shown. \textit{sil.} denotes the role silhouette score before projection, with higher values indicating stronger role separation.}
  \label{fig:layer-ffn-tsne}
\end{figure}
\FloatBarrier
\section{Additional Qualitative Comparisons}
\label{sec:appendix-qualitative}
Figures~\ref{fig:qualitative-chunkwise-supp-1}--\ref{fig:qualitative-chunkwise-supp-4} present twelve chunkwise cases, and Figures~\ref{fig:qualitative-framewise-supp-1}--\ref{fig:qualitative-framewise-supp-2} present six framewise cases, with matched prompts across SF, SGF, and SGF+.

\begin{figure}[!ht]
  \centering
  \includegraphics[width=\textwidth]{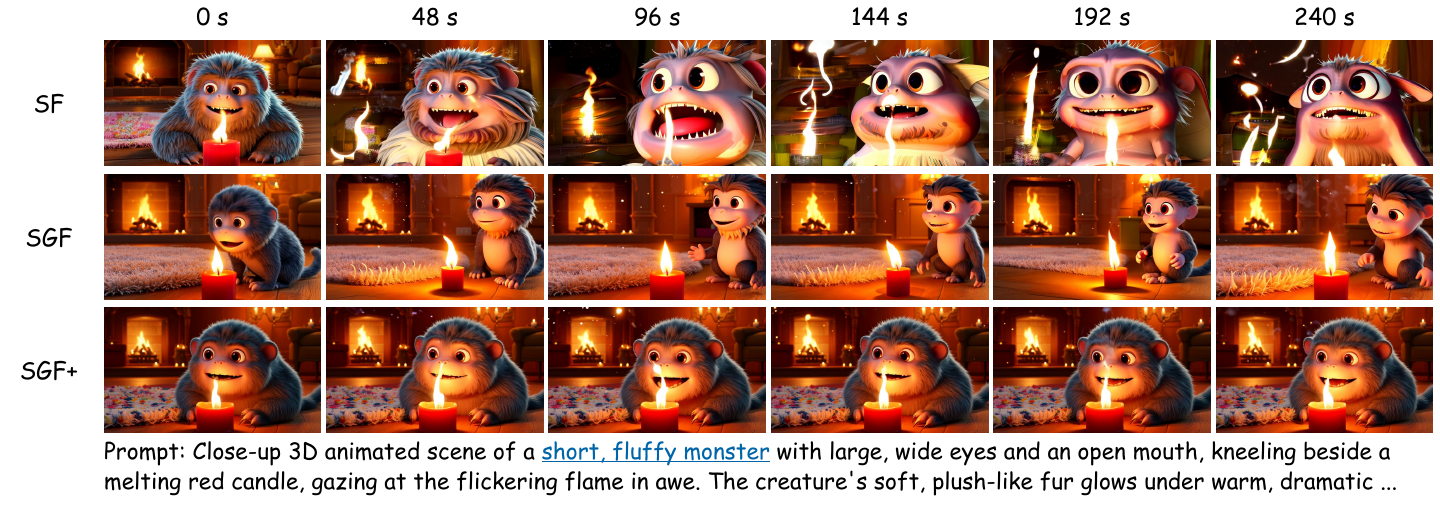}
  \vspace{2pt}
  \includegraphics[width=\textwidth]{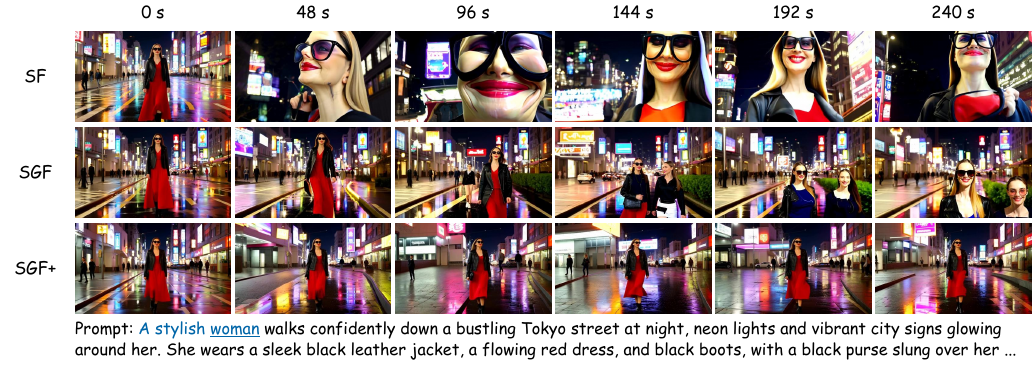}
  \vspace{2pt}
  \includegraphics[width=\textwidth]{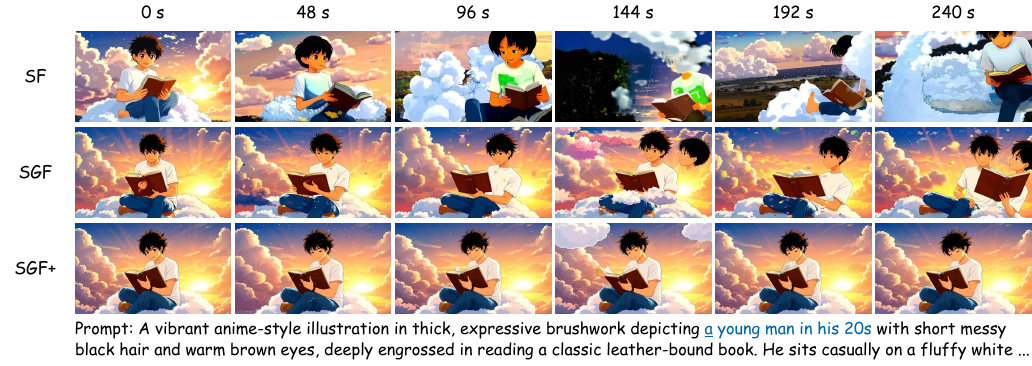}
  \caption{\textbf{Additional qualitative comparisons in chunkwise generation.} SF, SGF, and SGF+ are compared in three scenes. From top to bottom: a fluffy monster gazing at a candle, a neon-lit Tokyo street, and a person reading.}
  \label{fig:qualitative-chunkwise-supp-1}
\end{figure}

\begin{figure}[p]
  \centering
  \includegraphics[width=\textwidth]{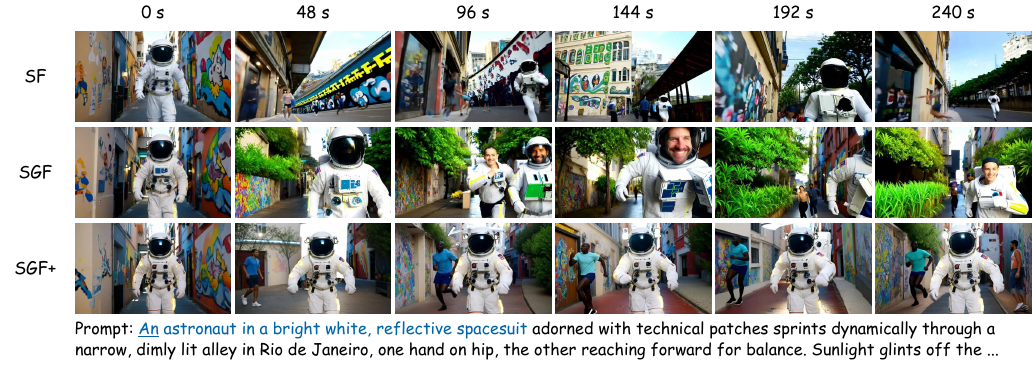}
  \vspace{2pt}
  \includegraphics[width=\textwidth]{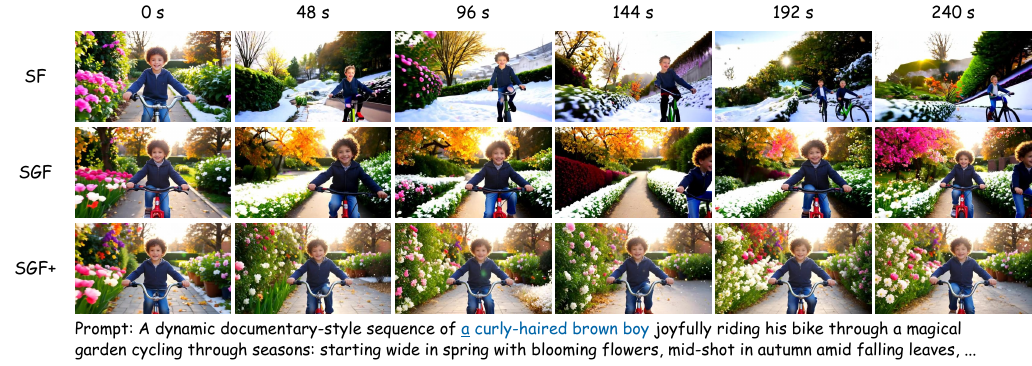}
  \vspace{2pt}
  \includegraphics[width=\textwidth]{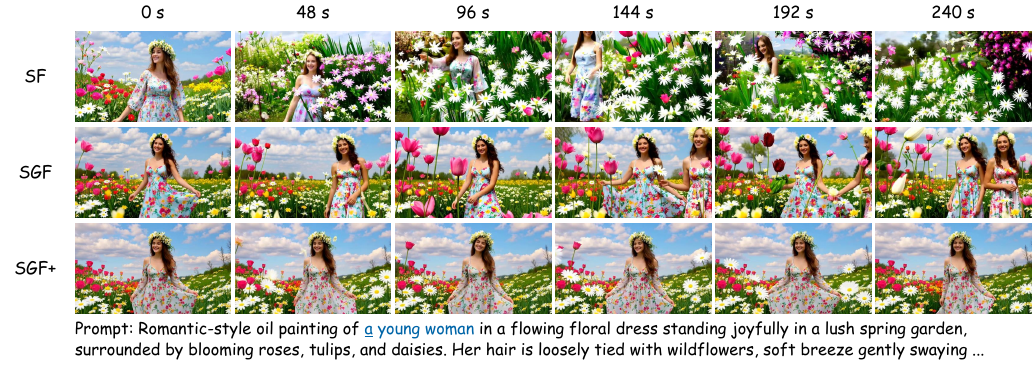}
  \caption{\textbf{Additional qualitative comparisons in chunkwise generation.} SF, SGF, and SGF+ are compared in three scenes. From top to bottom: an astronaut, a boy riding a bicycle, and a woman in a flower garden.}
  \label{fig:qualitative-chunkwise-supp-2}
\end{figure}

\begin{figure}[p]
  \centering
  \includegraphics[width=\textwidth]{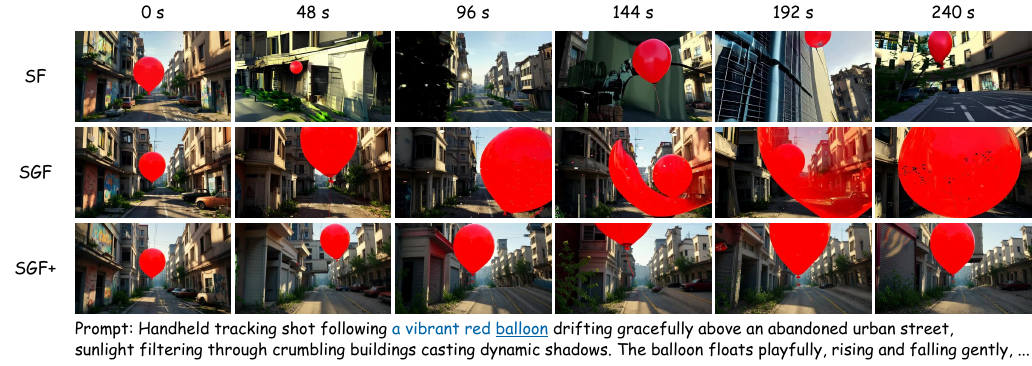}
  \vspace{2pt}
  \includegraphics[width=\textwidth]{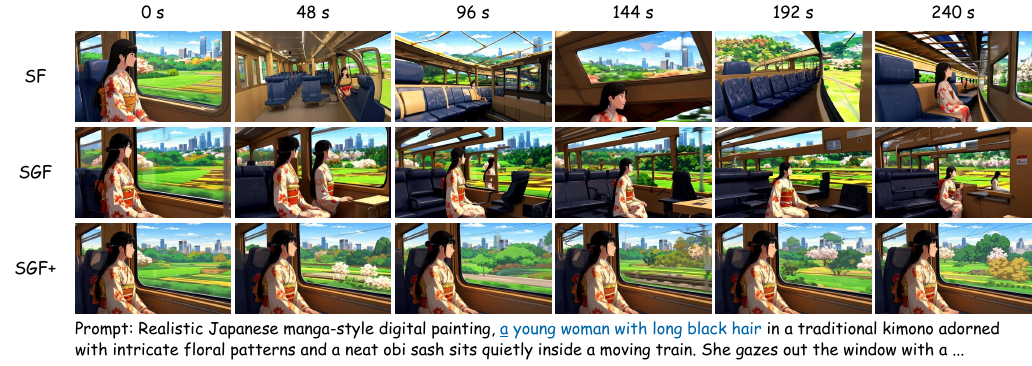}
  \vspace{2pt}
  \includegraphics[width=\textwidth]{figures/qualitative_chunkwise/chunkwise_240s_11.pdf}
  \caption{\textbf{Additional qualitative comparisons in chunkwise generation.} SF, SGF, and SGF+ are compared in three scenes. From top to bottom: a red balloon drifting through an urban street, a woman aboard a train, and a sunflower by a windowsill.}
  \label{fig:qualitative-chunkwise-supp-3}
\end{figure}

\begin{figure}[p]
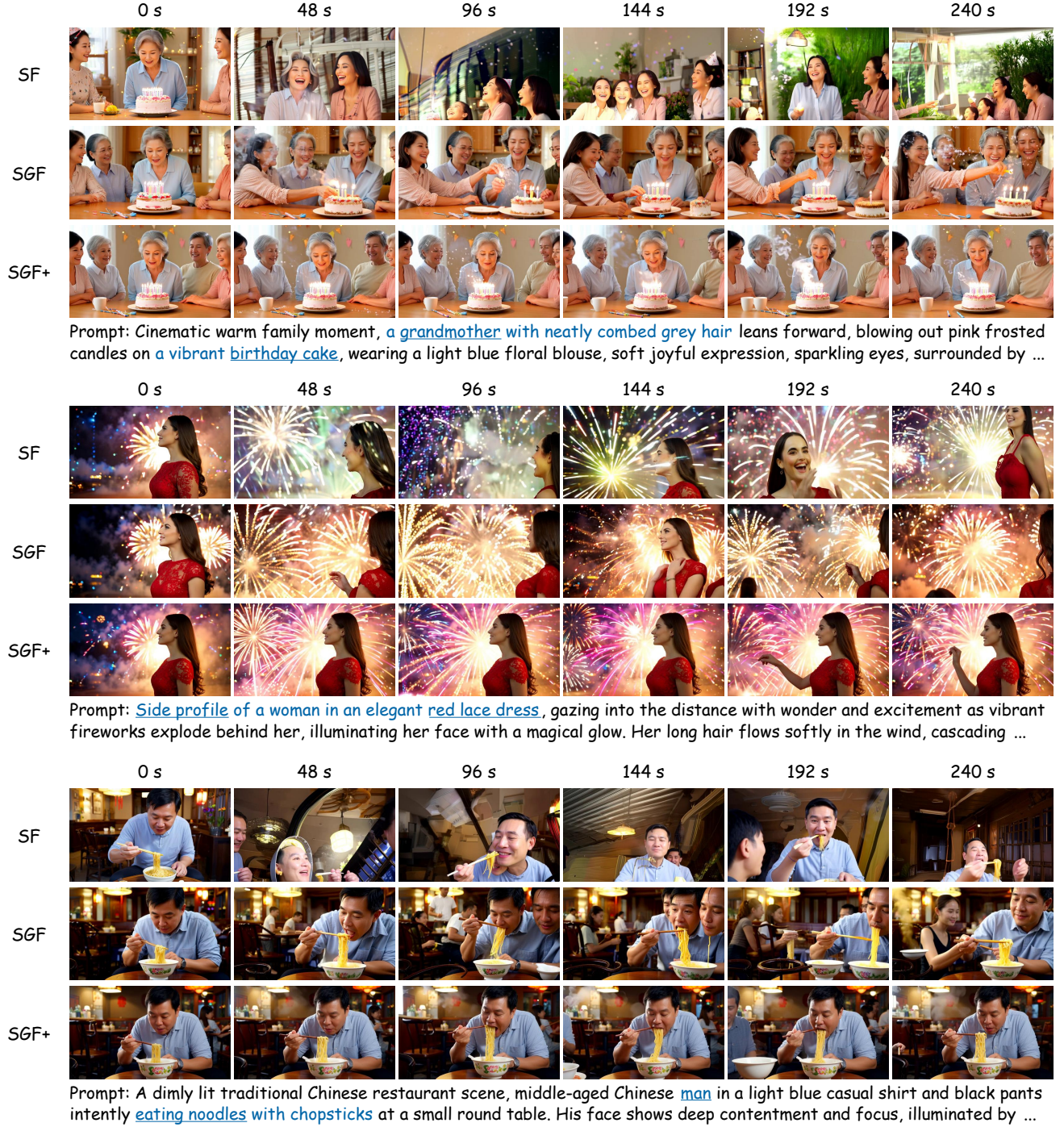

  \centering
  \includegraphics[width=\textwidth]{figures/qualitative_chunkwise/chunkwise_240s_12.pdf}
  \vspace{2pt}
  \includegraphics[width=\textwidth]{figures/qualitative_chunkwise/chunkwise_240s_13.pdf}
  \vspace{2pt}
  \includegraphics[width=\textwidth]{figures/qualitative_chunkwise/chunkwise_240s_14.pdf}
  \caption{\textbf{Additional qualitative comparisons in chunkwise generation.} SF, SGF, and SGF+ are compared in three scenes. From top to bottom: a grandmother celebrating her birthday, a woman in front of fireworks, and a man eating noodles with chopsticks.}
  \label{fig:qualitative-chunkwise-supp-4}
\end{figure}

\begin{figure}[p]
  \centering
  \includegraphics[width=\textwidth]{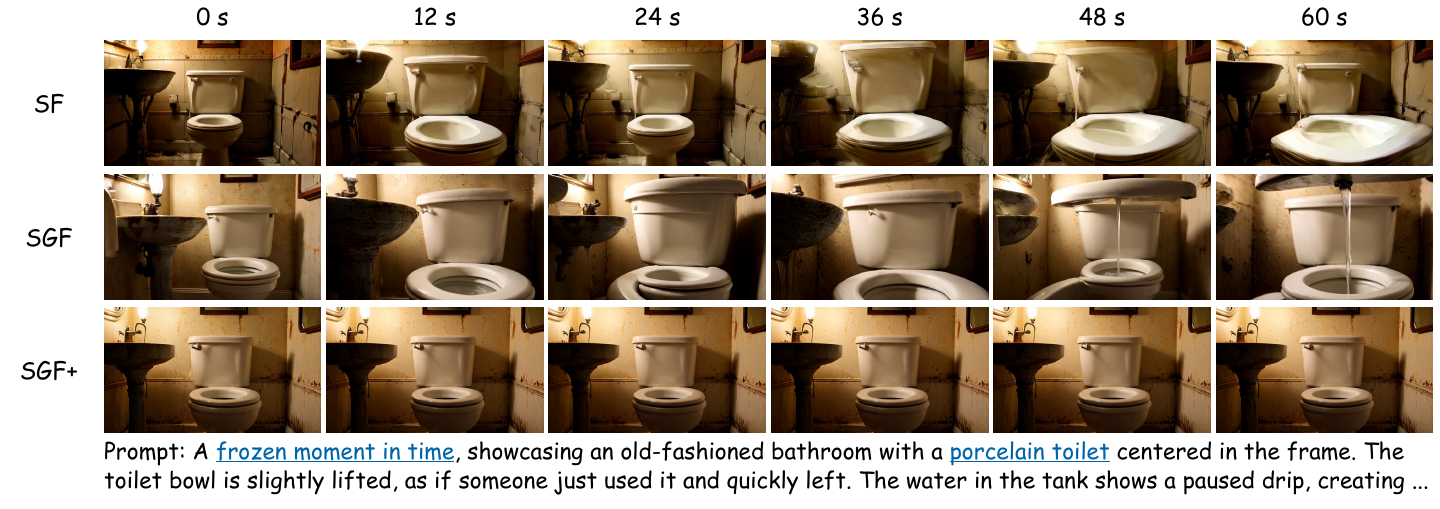}
  \vspace{2pt}
  \includegraphics[width=\textwidth]{figures/qualitative_framewise/framewise_240s_02.pdf}
  \vspace{2pt}
  \includegraphics[width=\textwidth]{figures/qualitative_framewise/framewise_240s_03.pdf}
  \caption{\textbf{Additional qualitative comparisons in framewise generation.} SF, SGF, and SGF+ are compared in three scenes. From top to bottom: a porcelain toilet in an old-fashioned bathroom, a white cat driving a red car, and a volcanic eruption in a coffee cup.}
  \label{fig:qualitative-framewise-supp-1}
\end{figure}

\begin{figure}[p]
  \centering
  \includegraphics[width=\textwidth]{figures/qualitative_framewise/framewise_240s_04.pdf}
  \vspace{2pt}
  \includegraphics[width=\textwidth]{figures/qualitative_framewise/framewise_60s_02.pdf}
  \vspace{2pt}
  \includegraphics[width=\textwidth]{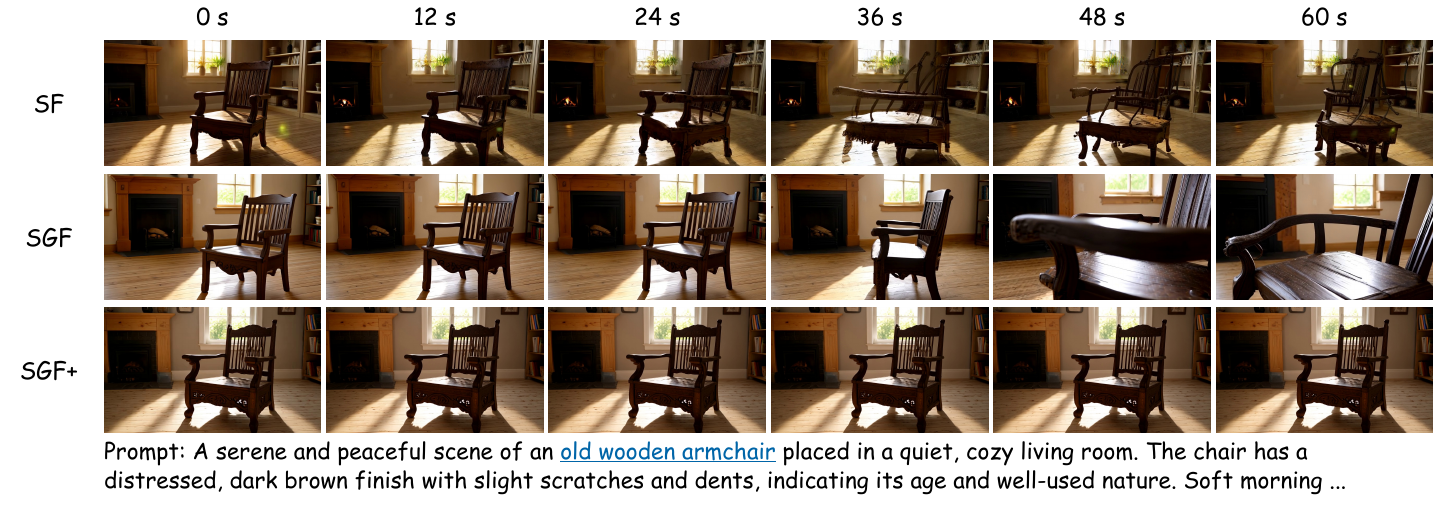}
  \caption{\textbf{Additional qualitative comparisons in framewise generation.} SF, SGF, and SGF+ are compared in three scenes. From top to bottom: a kangaroo dancing disco, Art Deco lampposts, and an old wooden armchair in a living room.}
  \label{fig:qualitative-framewise-supp-2}
\end{figure}

\FloatBarrier
\begingroup
\newcommand{\norm}[1]{\left\lVert#1\right\rVert}
\newcommand{\inner}[2]{\left\langle#1,#2\right\rangle}
\section{Local Optimization Analysis of Role Separation}
\label{sec:role-separation-theory}
We analyze the local effect of separating context-writing and denoising parameters in SGF+. Two comparison protocols distinguish the restriction imposed by parameter sharing from its effect on a gradient-descent step: a common update budget and a common learning rate.

\subsection{Fixed-replay objective and comparison geometry}
Fix the detached rollout latents, sampled denoising exits, DMD supervision, and all forward randomness. The resulting Pass-2 surrogate is
\begin{equation}
 F(\theta_c,\theta_d)
 =\ell\!\left(\mathcal D_{\theta_d}
 \left(Z^\star,t^\star\mid
 \mathcal C_{\theta_c}(\operatorname{sg}(X),0;\mathcal M_{\mathrm{rec}}),
 \mathcal M_{\mathrm{rec}}\right)\right)
 \label{eq:theory-objective}
\end{equation}
where $\ell$ uses the fixed supervision. Let $\theta_c,\theta_d\in\mathbb{R}^p$ have matching coordinates, such that tying them reproduces the shared forward computation. At $x_0=(\theta,\theta)$, define the full role-wise partial derivatives
\begin{equation}
 g_{\mathrm{C}}=\nabla_{\theta_c}F(x_0),\qquad
 g_{\mathrm{D}}=\nabla_{\theta_d}F(x_0),\qquad g=(g_{\mathrm{C}},g_{\mathrm{D}}).
\end{equation}
The writer derivative passes through the differentiable KV state. Both derivatives belong to the same objective, and the tied objective $f(\theta)=F(\theta,\theta)$ satisfies $\nabla f(\theta)=g_{\mathrm{C}}+g_{\mathrm{D}}$.

At this common point, impose the Euclidean product-space budget $\norm{\delta_c}^2+\norm{\delta_d}^2\le r^2$, with $r>0$. A tied update $(u,u)$ therefore costs $2\norm{u}^2$. For a subspace $V\subseteq\mathbb{R}^{2p}$ of admissible updates, define the maximum first-order decrease
\begin{equation}
 D_V(r)=\max_{\delta\in V,\,\norm{\delta}\le r}
             -\inner{g}{\delta}.
 \label{eq:theory-descent}
\end{equation}
This budget matches parameter displacement in the specified geometry; it does not match parameter count or computational cost.

\subsection{Descent lost under parameter sharing}
\begin{proposition}[Shared-parameter descent restriction]
\label{prop:sharing-gap}
Let $F$ be differentiable at $x_0$. For independent updates
$V_{\mathrm{split}}=\mathbb{R}^{2p}$ and tied updates
$V_{\mathrm{shared}}=\{(u,u):u\in\mathbb{R}^p\}$,
\begin{align}
 D_{\mathrm{split}}(r)&=r\sqrt{\norm{g_{\mathrm{C}}}^2+\norm{g_{\mathrm{D}}}^2},\\
 D_{\mathrm{shared}}(r)&=\frac{r}{\sqrt{2}}\norm{g_{\mathrm{C}}+g_{\mathrm{D}}}.
\end{align}
In particular,
\begin{equation}
 \boxed{D_{\mathrm{split}}(r)^2-D_{\mathrm{shared}}(r)^2
       =\frac{r^2}{2}\norm{g_{\mathrm{C}}-g_{\mathrm{D}}}^2.}
 \label{eq:theory-sharing-gap}
\end{equation}
\end{proposition}
\begin{proof}
Cauchy--Schwarz gives $-\inner{g}{\delta}\le r\norm{g}$, attained by
$\delta=-rg/\norm{g}$ when $g\ne0$. For tied updates, the budget becomes
$\norm{u}\le r/\sqrt{2}$ and the linear decrease is
$-\inner{g_{\mathrm{C}}+g_{\mathrm{D}}}{u}$, whose maximum is
$r\norm{g_{\mathrm{C}}+g_{\mathrm{D}}}/\sqrt{2}$. When either maximizing gradient vanishes, the corresponding maximum is zero. Subtracting the squared expressions and expanding the norms gives Eq.~\eqref{eq:theory-sharing-gap}.
\end{proof}

\subsection{Gradient cancellation and stationary points}
The orthogonal decomposition into shared and role-dependent directions is
\begin{equation}
 P_{\mathrm{shared}}g=
 \left(\frac{g_{\mathrm{C}}+g_{\mathrm{D}}}{2},\frac{g_{\mathrm{C}}+g_{\mathrm{D}}}{2}\right),\qquad
 g-P_{\mathrm{shared}}g=
 \left(\frac{g_{\mathrm{C}}-g_{\mathrm{D}}}{2},\frac{g_{\mathrm{D}}-g_{\mathrm{C}}}{2}\right).
 \label{eq:theory-projection}
\end{equation}
The second component, with squared norm $\norm{g_{\mathrm{C}}-g_{\mathrm{D}}}^2/2$, is excluded by parameter sharing.

When $E=\norm{g_{\mathrm{C}}}^2+\norm{g_{\mathrm{D}}}^2>0$, the fraction of squared gradient norm outside the tied-update subspace is
\begin{equation}
 \Gamma=\frac{\norm{g_{\mathrm{C}}-g_{\mathrm{D}}}^2}{2E}
       =1-\frac{D_{\mathrm{shared}}(r)^2}{D_{\mathrm{split}}(r)^2}
       =\frac12-\frac{\inner{g_{\mathrm{C}}}{g_{\mathrm{D}}}}{E}.
 \label{eq:theory-gamma}
\end{equation}
Thus $0\le\Gamma\le1$, with $\Gamma=0$ exactly when $g_{\mathrm{C}}=g_{\mathrm{D}}$ and $\Gamma=1$ exactly when $g_{\mathrm{C}}=-g_{\mathrm{D}}\ne0$. Negative alignment implies $\Gamma>1/2$, but any unequal role gradients yield $\Gamma>0$. Hence $\Gamma$ measures the sharing restriction, which is distinct from the negative-alignment criterion for gradient conflict.

For example, $g_{\mathrm{C}}=(2,1)$ and $g_{\mathrm{D}}=(2,-1)$ have cosine similarity $3/5$, while their second coordinates cancel in the shared gradient $(4,0)$. Nevertheless, $\Gamma=1/5>0$, illustrating a sharing restriction under positive overall alignment.

\begin{corollary}[Stationarity induced by cancellation]
\label{cor:cancellation}
If $g_{\mathrm{C}}=-g_{\mathrm{D}}\ne0$ at $x_0=(\theta,\theta)$, then $\theta$ is a stationary point of the tied objective $f$, whereas $x_0$ is not a stationary point of $F$. If, additionally, $\nabla F$ is $L$-Lipschitz on a neighborhood containing the segment from $x_0$ to $x_0-\eta g$, where $L>0$ and $0<\eta<2/L$, then
\begin{equation}
 F(x_0-\eta g)
 \le F(x_0)
 -\eta\left(1-\frac{L\eta}{2}\right)
       \left(\norm{g_{\mathrm{C}}}^2+\norm{g_{\mathrm{D}}}^2\right)
 < F(x_0).
 \label{eq:theory-cancellation-descent}
\end{equation}
\end{corollary}
\begin{proof}
The chain rule gives $\nabla f(\theta)=0$, whereas $\nabla F(x_0)=g\ne0$. The descent lemma along the update segment yields Eq.~\eqref{eq:theory-cancellation-descent}, whose decrease is strict for $0<\eta<2/L$.
\end{proof}

Complete cancellation is a sufficient condition for this stationary-point separation; negative alignment alone does not imply stationarity.

\subsection{Finite-step descent under a common update budget}
The first-order gap in Proposition~\ref{prop:sharing-gap} yields a finite-step advantage for the corresponding normalized updates, without requiring negative alignment.

\begin{corollary}[Finite-step advantage under a common budget]
\label{cor:common-budget-step}
Suppose $\nabla F$ is $L$-Lipschitz on an open neighborhood of the closed ball $B(x_0,\rho)$, with $L,\rho>0$, and $g_{\mathrm{C}}\ne g_{\mathrm{D}}$. Let
\begin{equation}
 a=\norm{g},\qquad b=\norm{P_{\mathrm{shared}}g},
 \qquad a>b\ge0.
\end{equation}
For $0<r\le\rho$, choose the updates that maximize first-order decrease under the product-space budget:
\begin{equation}
 \delta_{\mathrm{split}}^{(r)}=-\frac{r}{a}g,\qquad
 \delta_{\mathrm{shared}}^{(r)}=
 \begin{cases}
  -\dfrac{r}{b}P_{\mathrm{shared}}g,&b>0,\\
  0,&b=0.
 \end{cases}
 \label{eq:common-budget-updates}
\end{equation}
Then
\begin{equation}
 F(x_0+\delta_{\mathrm{shared}}^{(r)})
 -F(x_0+\delta_{\mathrm{split}}^{(r)})
 \ge r(a-b)-Lr^2.
 \label{eq:common-budget-bound}
\end{equation}
Consequently, for
\begin{equation}
 0<r<\min\left\{\rho,\frac{a-b}{L}\right\},
 \label{eq:common-budget-threshold}
\end{equation}
the split update strictly outperforms the shared update and strictly decreases the surrogate:
\begin{equation}
 F(x_0+\delta_{\mathrm{split}}^{(r)})
 <F(x_0+\delta_{\mathrm{shared}}^{(r)}),\qquad
 F(x_0+\delta_{\mathrm{split}}^{(r)})<F(x_0).
\end{equation}
\end{corollary}
\begin{proof}
Equation~\eqref{eq:theory-projection} gives $a^2-b^2=\norm{g_{\mathrm{C}}-g_{\mathrm{D}}}^2/2>0$. For $\norm{\delta}\le\rho$, smoothness yields
\begin{equation}
 F(x_0+\delta)=F(x_0)+\inner{g}{\delta}+R(\delta),
 \qquad |R(\delta)|\le\frac{L}{2}\norm{\delta}^2.
 \label{eq:theory-taylor}
\end{equation}
The two updates have norm at most $r$ and linear terms $-ra$ and $-rb$, respectively, including $b=0$. Subtracting their expansions bounds the remainder difference by $Lr^2$, giving Eq.~\eqref{eq:common-budget-bound}. Equation~\eqref{eq:common-budget-threshold} makes this bound positive and also ensures
\begin{equation}
 F(x_0+\delta_{\mathrm{split}}^{(r)})
 \le F(x_0)-ra+\frac{Lr^2}{2}<F(x_0),
\end{equation}
because $r<(a-b)/L\le a/L<2a/L$.
\end{proof}

For $b>0$, both steps have product-space norm $r$; for $b=0$, the shared step is zero. The comparison concerns the normalized updates in Eq.~\eqref{eq:common-budget-updates}.

\subsection{Negative alignment under a common learning rate}
\label{sec:equal-step-theory}
We next compare ordinary gradient-descent steps with a common learning rate $\eta$ for the shared parameter vector and each independent role vector.

Let $h=g_{\mathrm{C}}+g_{\mathrm{D}}$, $E=\norm{g_{\mathrm{C}}}^2+\norm{g_{\mathrm{D}}}^2$, and $S=\norm{h}^2$. In the product space, the updates are
\begin{equation}
 \delta_{\mathrm{shared}}=-\eta(h,h),\qquad
 \delta_{\mathrm{split}}=-\eta(g_{\mathrm{C}},g_{\mathrm{D}}).
 \label{eq:equal-step-updates}
\end{equation}
Their squared norms are $2\eta^2S$ and $\eta^2E$, respectively, so this protocol generally uses different displacement budgets.

\begin{proposition}[One-step advantage under negative alignment]
\label{prop:negative-alignment}
Suppose $\nabla F$ is $L$-Lipschitz on an open neighborhood of the closed ball $B(x_0,\rho)$, with $L,\rho>0$. If $\inner{g_{\mathrm{C}}}{g_{\mathrm{D}}}=-\kappa<0$, set
$K=E+2S$ and $M=\max\{\sqrt{E},\sqrt{2S}\}$. For $\eta>0$ with $\eta M\le\rho$,
\begin{equation}
 \left|F(x_0+\delta_{\mathrm{shared}})
       -F(x_0+\delta_{\mathrm{split}})-2\eta\kappa\right|
 \le \frac{L\eta^2}{2}K.
 \label{eq:equal-step-bound}
\end{equation}
Consequently, whenever
\begin{equation}
 0<\eta<\min\left\{\frac{\rho}{M},\frac{4\kappa}{LK},\frac{2}{L}\right\},
 \label{eq:equal-step-threshold}
\end{equation}
the split update both outperforms the shared update and strictly decreases the surrogate:
\begin{equation}
 F(x_0+\delta_{\mathrm{split}})
 <F(x_0+\delta_{\mathrm{shared}}),\qquad
 F(x_0+\delta_{\mathrm{split}})<F(x_0).
 \label{eq:equal-step-order}
\end{equation}
\end{proposition}
\begin{proof}
Both updates lie in the smoothness ball, so the expansion in Eq.~\eqref{eq:theory-taylor} applies. The linear terms are $-\eta S$ and $-\eta E$, with difference
$\eta(E-S)=-2\eta\inner{g_{\mathrm{C}}}{g_{\mathrm{D}}}=2\eta\kappa$.
The sum of the remainder bounds is $L\eta^2(E+2S)/2$, proving Eq.~\eqref{eq:equal-step-bound}. In particular,
\begin{equation}
 F(x_0+\delta_{\mathrm{shared}})-F(x_0+\delta_{\mathrm{split}})
 \ge 2\eta\kappa-\frac{L\eta^2}{2}K>0
\end{equation}
when $\eta<4\kappa/(LK)$. Finally,
$F(x_0+\delta_{\mathrm{split}})\le F(x_0)-\eta(1-L\eta/2)E<F(x_0)$
for $\eta<2/L$. Negative alignment ensures $E,K,M>0$, so the interval in Eq.~\eqref{eq:equal-step-threshold} is nonempty.
\end{proof}

Under the same smoothness assumptions, the common-learning-rate comparison has the general expansion
\begin{equation}
 F(x_0+\delta_{\mathrm{shared}})
 -F(x_0+\delta_{\mathrm{split}})
 =-2\eta\inner{g_{\mathrm{C}}}{g_{\mathrm{D}}}+O(\eta^2).
 \label{eq:equal-step-general-alignment}
\end{equation}
A strictly positive inner product therefore favors the shared update for sufficiently small $\eta$, even if some coordinates cancel. This differs from Corollary~\ref{cor:common-budget-step} because the common-learning-rate protocol does not normalize update lengths.

For a minibatch, the proposition requires negative alignment of the aggregated gradients. Negative alignment of individual samples or selected modules does not establish this condition for the full gradient. A selected-block comparison holds with all other parameters fixed.

\subsection{Partial separation across network modules}
Partition the matching role parameters into $B$ disjoint blocks, with gradients $g_{\mathrm{C}}^{(b)}$ and $g_{\mathrm{D}}^{(b)}$. Blocks in $\mathcal S\subseteq\{1,\ldots,B\}$ have independent role updates; the remaining blocks are tied. All blocks share the same total product-space budget.

\begin{corollary}[Residual restriction under partial separation]
\label{cor:partial-separation}
At the common shared point, the maximum first-order decrease $D_{\mathcal S}(r)$ for this partially separated update space satisfies
\begin{equation}
 D_{\mathcal S}(r)^2=r^2\left[
 \sum_{b\in\mathcal S}\left(\norm{g_{\mathrm{C}}^{(b)}}^2+\norm{g_{\mathrm{D}}^{(b)}}^2\right)
 +\frac12\sum_{b\notin\mathcal S}\norm{g_{\mathrm{C}}^{(b)}+g_{\mathrm{D}}^{(b)}}^2
 \right].
 \label{eq:theory-partial-descent}
\end{equation}
Consequently, relative to full separation,
\begin{equation}
 \boxed{D_{\mathrm{full}}(r)^2-D_{\mathcal S}(r)^2
       =\frac{r^2}{2}\sum_{b\notin\mathcal S}
                    \norm{g_{\mathrm{C}}^{(b)}-g_{\mathrm{D}}^{(b)}}^2.}
 \label{eq:theory-partial-gap}
\end{equation}
\end{corollary}
\begin{proof}
For any linear subspace $V$ and orthogonal projection $P_V$,
$\inner{g}{\delta}=\inner{P_Vg}{\delta}$ for $\delta\in V$. Thus
$D_V(r)=r\norm{P_Vg}$, including the zero-projection case. In a separated block the projection retains both gradient components. In a tied block it replaces them by their average, as in Eq.~\eqref{eq:theory-projection}, with squared norm
$\norm{g_{\mathrm{C}}^{(b)}+g_{\mathrm{D}}^{(b)}}^2/2$.
Disjoint parameter blocks are orthogonal, so their squared projection norms add, giving Eq.~\eqref{eq:theory-partial-descent}. Subtracting from
$D_{\mathrm{full}}(r)^2=r^2\sum_b(\norm{g_{\mathrm{C}}^{(b)}}^2+\norm{g_{\mathrm{D}}^{(b)}}^2)$
and applying Eq.~\eqref{eq:theory-sharing-gap} blockwise proves the result.
\end{proof}

Separating only FFN leaves the Attention contribution in Eq.~\eqref{eq:theory-partial-gap}. This quantifies the residual sharing restriction in the partial-separation variants, provided that tying their matching parameters reproduces the same forward computation at $x_0$.

\endgroup

\section{Future Work}
\label{sec:future-work}

\noindent\textbf{World action models.}\enspace
Our observation of gradient conflict at TF initialization motivates exploring role-specific parameterization from the teacher-forcing stage. World action models (WAMs)~\cite{ye2026dreamzero,li2026lingbotva,wang2026repwam,peng2026fact}, which commonly rely on teacher-forced training, provide a natural setting for this direction. Separating context writing from visual and action prediction may enable more effective learning of historical representations for subsequent predictions.

\par\smallskip
\noindent\textbf{Efficient context writers and KV compression.}\enspace
Role-specific parameterization motivates asymmetric architectures, with model capacity tailored to the demands of context writing and denoising. A lightweight context writer could learn compressed KV representations from redundant video histories under a constrained cache budget. Such writer-side adaptation could be optimized through future generation losses while keeping the denoiser's parameters and attention interface unchanged. The goal is to reduce parameter storage, cache memory, and attention cost while preserving the denoiser's generative capabilities and long-horizon consistency.

\end{document}